\documentclass{article}

\usepackage{times}

\usepackage{amsthm}
\usepackage{mathtools}
\usepackage{wrapfig}

\usepackage[utf8]{inputenc} 
\usepackage[T1]{fontenc}    
\usepackage{hyperref}       
\usepackage{url}            
\usepackage{booktabs}       
\usepackage{amsfonts,dsfont,amsmath}       
\usepackage{nicefrac}       
\usepackage{microtype}      

\title{Robustness of Community Detection to Random Geometric Perturbations}

\author{%
  Sandrine~P\'ech\'e \\ 
  LPSM, University Paris Diderot\\
  \texttt{peche@lpsm.fr}\\
\and
  Vianney~Perchet\\
  Crest, ENSAE \& Criteo AI Lab\\
 \texttt{vianney.perchet@normalesup.org} \\}

\usepackage{amsfonts, bbm}
\usepackage{psfrag, color, tikz,dsfont}
\usetikzlibrary{shapes}
\usepackage{cite}

\newcommand{\R}{\ensuremath{\mathds{R}}}
\newcommand{\C}{\ensuremath{\mathbb{C}}}

\newcommand{\mbE}{\ensuremath{\mathds{E}}}

\newcommand{\bpm}{\begin{pmatrix}}
\newcommand{\epm}{\end{pmatrix}}

\newtheorem{theorem}{Theorem}

\newtheorem{lemma}[theorem]{Lemma}
\newtheorem{proposition}[theorem]{Proposition}
\newtheorem{conjecture}[theorem]{Conjecture}

\newcommand{\brem}{\begin{remark}}
\newcommand{\erem}{\end{remark}}
\newcommand{\bconj}{\begin{conjecture}}
\newcommand{\econj}{\end{conjecture}}
\newcommand{\bdefi}{\begin{definition}}
\newcommand{\edefi}{\end{definition}}
\newcommand{\bt}{\begin{theo}}
\newcommand{\bfa}{\begin{fact}}
\newcommand{\efa}{\end{fact}}
\newcommand{\becor}{\begin{coro}}
\newcommand{\ecor}{\end{coro}}
\newcommand{\et}{\end{theo}}
\newcommand{\bp}{\begin{proposition}}
\newcommand{\ep}{\end{proposition}}
\newcommand{\bl}{\begin{lemma}}
\newcommand{\el}{\end{lemma}}
\newcommand{\be}{\begin{equation}}
\newcommand{\ee}{\end{equation}}
\newcommand{\mbP}{\mathbb{P}}

\newcommand{\Tr}{\mathrm{Tr}}

\newcommand{\E}{\mathbb{E}}

  \newcommand\smallO{o}

\DeclareMathOperator{\sign}{sign}

\begin{document}

\maketitle

\begin{abstract}
We consider the stochastic block model where connection between vertices is perturbed by some latent (and unobserved) random geometric graph. The objective is to prove that spectral methods are robust to this type of noise, even if they are agnostic to the presence (or not) of the random graph. We provide explicit regimes where the  second eigenvector of the adjacency matrix is highly correlated to the true community vector (and therefore when weak/exact recovery is possible). This is possible thanks to a detailed analysis of the spectrum of the latent random graph, of its own interest.
\end{abstract}

\section*{Introduction}
In a  $d$-dimensional  random  geometric  graph, $N$ vertices are assigned random coordinates in  $\R^d$, and only points close enough to each other are connected by an edge. Random geometric graphs are used to model complex networks such as  social networks, the world wide web and so on. We refer to \cite{penrose2003random}  - and references therein - for a comprehensive introduction to random geometric graphs. On the other hand, in social networks, users are more likely to connect if they belong to some specific community (groups of friends, political party, etc.). This has motivated the introduction of the stochastic block models (see the recent survey \cite{abbe2017community} and the more recent breakthrough \cite{bordenave2015} for more details), where in the simplest case, each of the $N$ vertices belongs to one (and only one) of the two communities that are present in the network.

The two types of connections -- geometric graph vs.\ block model -- are conceptually quite different and co-exist independently. Two users might be connected because they are ``endogenously  similar'' (their latent coordinates are close enough to each others) or because they are ``exogenously similar'' (they belong to the same community). For instance, to oversimplify a social network, we can consider that two different types of connections can occur  between  users: either they are childhood friends  (with similar latent variables) or they have the same political views (right/left wing).

We therefore model these simultaneous  types of interaction in social networks as a simple stochastic block model (with 2 balanced communities) perturbed by a latent geometric graph.  More precisely, we are going to assume that the probability of endogenous connections between vertices $i$ and $j$, with respective latent variables $X_i$, $X_j \in \R^d$, is given by the Gaussian\footnote{We emphasize here that  geometric interactions are defined through some kernel so that different recovery regimes can be identified with respect to a unique, simple width parameter $\gamma$. Similarly,  the choice of the Gaussian kernel might seem a bit specific and arbitrary, but this  purely for the sake of presentation: our approach can be generalized to other kernels (the ``constants'' will be different; they are defined w.r.t.\ the kernel chosen).
} kernel $\exp(-\gamma \|X_i-X_j\|^2)$ where $\gamma$ is the (inverse) width. On the other hand, exogenous connections are defined by the block model where half of the $N$ vertices belong to some community, half of them to the other one.  The probability of connection between two members of the same community is equal to $p_1$ and between two members from different communities is equal to $p_2$. We also consider an extra parameter $\kappa \in [0,1]$ to represent the respective strengths of exogenous vs.\ endogenous connections (and we assume that $\kappa + \max\{p_1,p_2\} \leq 1$ for technical reason).

Overall, the probability of connection between $i$ and $j$, of latent variable $X_i$ and $X_j$ is 
$$\mathds{P}\big\{ i \sim j \, \big|\, X_i, X_j \big\} = \kappa e^{-\gamma \|X_i-X_j\|^2} +\begin{cases} p_1 &\text{ if $i,j$ are in the same community}\cr
p_2& \text{ otherwise.}\end{cases}$$

In  stochastic block models, the key idea is to recover  the two communities from the observed set of edges (and only from those observations, i.e., the latent variables $X_i$ are not observed). This recovery can have different variants that we enumerate now (from the strongest to the weakest). Let us denote by $\sigma \in \{\frac{\pm 1}{\sqrt{N}}\}^N$ the normalized community vector illustrating to which community each vertex belong ($\sigma_i = -\frac{1}{\sqrt{N}}$ if $i$ belongs the the first community and $\sigma_i = \frac{1}{\sqrt{N}}$ otherwise).

Given the graph-adjency matrix $A \in \{0,1\}^{N^2}$, the objective is to output a normalized vector $x \in \R^N$ (i.e., with $\|x\|=1$) such that, for some $\varepsilon >0$,

\medskip

\begin{tabular}{rl}
\textbf{Exact recovery:}&with probability tending to 1, $\big|\sigma^\top x\big|=1$, thus $x \in  \{\frac{\pm 1}{\sqrt{N}}\}^N$\\
\textbf{Weak recovery:}&with probability tending to 1,   $\big|\sigma^\top x\big|\geq \varepsilon$ and $x \in  \{\frac{\pm 1}{\sqrt{N}}\}^N$\\
\textbf{Soft recovery:}&with probability tending to 1,   $\big|\sigma^\top x\big|\geq \varepsilon$\end{tabular}
\medskip

%
We recall here that if $x$ is chosen at random, independently from $\sigma$, then $\big|\sigma^\top x\big|$ would be of the order of $\frac{1}{\sqrt{N}}$, thus tends to 0. On the other hand, weak recovery implies that the vector $x$ has (up to a change of sign)  at least $\frac{N}{2}(1+\varepsilon)$ coordinates equal to those of $\sigma$. Moreover, we speak of soft recovery (as opposed to hard recovery) in the third case by analogy to soft vs.\ hard classifiers. Indeed, given any normalized vector $x \in \R^d$, let us construct the vector $\sign(x)=\big(\frac{2\mathds{1}\{X_i\geq 0\}-1 }{\sqrt{N}}\big) \in \{\frac{\pm 1}{\sqrt{N}}\}^N$.  Then $\sign(x)$ is a candidate for weak/exact recovery. Standard comparisons between Hamming and Euclidian distance (see, e.g., \cite{lelarge2015reconstruction}) relates soft to weak recovery as 
$$\big| \sigma^\top \sign(x) \big| \geq 4 \big|\sigma^\top x\big| -3;$$ 
In particular, weak-recovery is ensured as soon as soft recovery is attained above the threshold of $\varepsilon=3/4$ (and obviously exact recovery after the threshold $1-1/4N$).

\medskip

For simplicity, we are going to assume\footnote{The fact that $d=2$ does not change much compared to $d>3$; it is merely for the sake of  computations; any Gaussian distribution $\mathcal{N}(0,\sigma^2I_2)$  can be recovered by dividing $\gamma$ by $\sigma^2$.} that $X_i$ are i.i.d., drawn from the 2-dimensional Gaussian distribution $\mathcal{N}(0,I_2)$. In particular, this implies that the law $A_{i,j}$ (equal to 1 if there is an edge between $i$ and $j$  and 0 otherwise) is a Bernoulli random variable (integrated over $X_i$ and $X_j$) $
\mathrm{Ber}\Big(\frac{p_1+p_2}{2}+ \frac{\kappa}{1+4\gamma}\Big);
$
Notice that $A_{i,j}$ and $A_{i',j'}$ are identically distributed but not independent if $i=i'$ or $j=j'$. Recovering communities can be done efficiently (in some regime) using spectral methods and we will generalize them to this perturbed  (or mis-specified) model. For this purpose, we will need a precise and detailed spectral analysis of the random geometric graphs considered (this has been initiated in \cite{rai2007spectrum}, \cite{dettmann2017spectral} and \cite{blackwell2007spectra} for instance).

\medskip

There has been several extensions of the standard stochastic block models to incorporate latent variables or covariables in perturbed stochastic block models. We can mention cases where  covariables are observed (and thus the algorithm can take their values into account to optimize the community recovery) \cite{yan2019covariate,weng2016community,deshpande2018contextual,huang2018pairwise}, when the degree of nodes are corrected \cite{gao2018community} or the case of labeled edges \cite{heimlicher2012community,xu2014edge,jog2015information,lelarge2015reconstruction,yun2016optimal}.  However, these papers do not focus on the very simple question of the robustness of recovery algorithm to (slight) mis-specifications in the model, i.e., to some small perturbations of the original model and this is precisely our original  motivations. Regarding this question, \cite{Massoulie} consider the robustness of spectral methods for a SBM perturbed by adversarial perturbation in the sparse degree setup.  Can we prove that a specific efficient algorithm (here, based on spectral methods) still exactly/weakly/softly recover communities even if it is agnostic to the presence, or not, of endogenous  noise  ? Of course, if that noise is too big, then recovery is impossible (consider for instance the case $\gamma =0$ and $\kappa\gg 0$). However, and this is our main contribution, we are able to pinpoint specific range of perturbations (i.e., values of $\kappa$ and $\gamma$) such that spectral methods -- in short, output the normalized second highest eigenvector -- still manage to perform some recovery of the communities. Our model is motivated to simplify the exposition but can be generalized to more complicated models (more than two communities of different sizes).

To be more precise, we will prove that:\\
- if $1/\gamma$ is in the same order than $p_1$ and $p_2$ (assuming that $p_1 \sim p_2$ is a  standard assumption in stochastic block model), then soft recovery is possible under a mild assumption ($\frac{p_1 -p_2}{2} \geq 4 \frac{\kappa}{\gamma}(1+\varepsilon)$); \\
- if $\gamma(p_1-p_2)$ goes to infinity, then exact recovery happens. \\
However, we mention here that we do not consider the ``sparse'' case (when $p_i \sim \frac{a}{n}$), in which  regimes where partial recovery is possible or not (and efficiently) are now clearly understood \cite{decelle2011asymptotic,massoulie2014community,deshpande2015asymptotic,mossel2015reconstruction}, as the geometric graphs perturbes too much the delicate arguments.

Our main results are summarised in Theorem \ref{TH:known} (when the different parameters are given) and Theorem \ref{Th:mainSBM} (without knowing them, the most interesting case). It is a first step for the study of the robustness of spectral methods in the presence of endogenous noise regarding the question of community detection.\\ 
As mentioned before, those results highly rely on a careful and detailed analysis of the spectrum of the random graph adjencency matrix. This is the purpose of the following Section \ref{SE:spectral}, which has its own interest in random graphs. Then we investigate the robustness of spectral methods in a perturbed stochastic block model, which is the main focus of the paper, in Section \ref{SE:SBM}. Finally, more detailed analysis, other statements and some proofs are given in the Appendix.
\vspace{-0.3cm}

\section{Spectral analysis for the adjacency  matrix of the random grah}\label{SE:spectral}
Let us denote by $P$ the \emph{conditional expectation matrix} (w.r.t\  the Gaussian kernel), where $P_{ij}=P_{ji}=e^{-\gamma ||X_i-X_j||^2}$, for $i<j \in [1, .., N],$
 and $P_{ii}=0$ for all  $i=1, ..,N$. We will denote by $\mu_1\geq \mu_2\geq \cdots \geq \mu_N$ its ordered eigenvalues (in Section \ref{SE:SBM}, $\mu_k$ are the eigenvalues of $\kappa P$).


\subsection{The case where $\gamma$ is bounded}


We study apart the case where $\limsup_{N\to \infty} \gamma <\infty$. 
The simplest case corresponds to the case where $\gamma  \log(N) \to 0$ as $N \to \infty$ as with probability one, each $P_{i,j}$ converges to one.  And as a consequence,   the spectrum  of $P$ has a nonzero eigenvalue which converges to $N$ (with probability arbitrarily close to $1$).  In the case where $\gamma$ is not negligible w.r.t. $\frac{1}{\log(N)}$,  arguments to understand the spectrum of $P$ -- or at least its spectral radius -- are a bit more involved.
\bp Assume that $\gamma(N)$ is a sequence such that $\lim_{N\to \infty} \gamma(N)=\gamma_0\geq 0$. Then there exists a constant $C_1(\gamma_0)$ such that the largest eigenvalue of $P$ satisfies 
$$\frac{\mu_1(P)}{N C_1(\gamma_0)} \to 1 \mathrm{\ as\ } N \to \infty.$$
\label{lemme:gammaconstant}
\ep
\vspace{-0.5cm}

\subsection{The spectral radius of $P$ when $\gamma \to \infty$, $\gamma \ll  N/\ln N$ }
We now investigate the special case where $\gamma \to \infty$, but when $\gamma \ll  N/\ln N$ (as in this regime  the spectral radius $\rho(P)$ of $P$ does not vanish). We  will show that $\rho(P)$ is in the order of $\frac{N}{2\gamma}$.

We formally state this case under the following Assumption \eqref{H0} (implying that  $\gamma \ln \gamma \ll N$).
 \be \label{H0} \tag{$H_1$} \: \gamma \to \infty \ \mathrm{\ and\ } \ \frac{1}{\gamma}\frac{N}{\ln N} \to \infty. \ee

\bp If Assumption (\ref{H0}) holds then, with probability tending to one,
$$  \frac{N}{2\gamma} \leq  \rho(P) \leq \frac{N}{2\gamma} \left (1+ \smallO(1)\right ).$$
\label{Prop1}
\ep
\vspace{-0.5cm}

\begin{proof}
By the Perron Frobenius theorem,
one has that 
$$\min_{i=1, \ldots, N} \sum_{l=1}^N P_{il} \leq \rho(P) \leq \max_{i=1, \ldots, N} \sum_{l=1}^N P_{il}.$$ To obtain an estimate of the spectral radius of $P$, we  show that, with probability tending to 1,   $\max_{i} \sum_{l=1}^N P_{il} $ cannot exceed $\frac{N}{2\gamma}$ and  for ``a large enough number" of indices $i$, their \emph{connectivity} satisfies 
$$ \sum_{l=1}^N P_{il} = \frac{N}{2\gamma}\left (1+ o(1)\right ).
$$
The proof is going to be decomposed into three parts (each corresponding to a different lemma, whose proofs are delayed to Appendix \ref{App:spectral}.).
\begin{enumerate}
\item We first consider only vertices close to 0, i.e., such that $|X_i|^2 \leq 2\frac{\log(\gamma)}{\gamma}$. For those vertices,  $\sum_j P_{i,j}$ is of the order of  $ N/2\gamma$ with probability close to 1. See Lemma \ref{lem: 1}
\item For the other vertices, farther away from 0, it is easier to only provide an upper bound on $\sum_j P_{i,j}$ with a similar proof.  See Lemma \ref{lemme:ub}
\item  Then we show that the spectral radius has to be of the order $N/2\gamma$  by considering  the subset $J$ of vertices  "close to 0" (actually introduced in the first step) and by proving that their inner connectivity -- restricted to $J$ --, must be of the order $N/2\gamma$. See Lemma \ref{lem: UB}.
\end{enumerate}
Combining the following three Lemmas \ref{lem: 1},  \ref{lemme:ub} and \ref{lem: UB}  will immediately give the result. 
\end{proof}

\bl \label{lem: 1} Assume that Assumption (\ref{H0}) holds, then, as $N$ grows to infinity,
$$
\mbP \Big\{ \exists i \leq N \mathrm{\ s.t.\ } |X_i|^2 \leq 2\frac{\ln \gamma}{\gamma},  \Big|\sum_{j=1}^N P_{ij} -\frac{N}{2\gamma}\Big| \leq \smallO\Big(\frac{N}{2\gamma}\Big)\Big\}\to 1.$$
\el 
\vspace{-0.4cm}

Lemma \ref{lem: 1} states that the connectivities of vertices close to the origin  converge to their expectation (conditionally to  $X_i$). Its proof  decomposes the set of vertices into those that are close to $i$ (the main contribution in the connectivity, with some concentration argument), far from $i$ but close to the origin (negligible numbers) and those far from $i$ and the origin (negligible contribution to the connectivity).

The second step of the proof of Proposition \ref{Prop1} considers  indices $i$  such that $|X_i|^2\geq 2\frac{\ln \gamma}{\gamma}$.
\bl For indices $i$ such that $|X_i|^2\geq 2\frac{\ln \gamma}{\gamma}$ one has with  probability tending to 1  that $$\sum_{j=1}^N P_{ij}\leq \frac{N}{2\gamma}\left (1+\smallO(1)\right).$$\label{lemme:ub}
\el
\vspace{-0.6cm}
The proof just uses the fact that for those vertices, $P_{ij}$ are typically negligible.

To get a lower bound on the spectral radius of $P$, we show that if one selects the submatrix $P_J:=(P_{ij})_{i,j \in J}$ where $J$ is the collection of indices 
\be \label{J}J=\Big \{1\leq i \leq N,  |X_i|^2 \leq 2\frac{\ln \gamma }{\gamma}\Big \},\ee
the spectral radius of $P_J$ is almost $ \frac{N}{2\gamma}$. This will give the desired estimate on the spectral radius of $P$. 

\bl Let $J$ be the subset defined in Equation \eqref{J} and $P_J$ the associated sub matrix. Let $\mu_1(J)$ denote the largest eigenvalue of $P_J$.
Then, with h.p., one has that 
$$\mu_1(J)\geq \frac{N}{2\gamma} (1-\smallO(1)).$$
\label{lem: UB}
\el 
\vspace{-0.5cm}
The proof relies on the fact that vertices close to the origin get the most contribution to their connectivity from the other vertices close to the origin.

\medskip

The constant $1/2$ that arises in the Proposition \ref{Prop1} is a direct consequence of the choice of the Gaussian kernel. Had we chosen a different kernel, this constant would have been different (once the width parameter $\gamma$ normalized appropriately). The techniques we developed can be used to compute it; this is merely a matter of computations, left as exercices.
\vspace{-0.2cm}
\section{A stochastic block model perturbed by a geometric graph}\label{SE:SBM}

\subsection{The model }We  consider in this section the stochastic block model, with two communities (it can easily be extended to the coexistence of more communities), yet perturbed by a geometric graph. More precisely, we assume that each member $i$ of the network (regardless of its community) is characterized by an i.i.d.\  Gaussian vector $X_i$ in $\R^2$ with distribution $\mathcal{N}(0,I_2)$.

The perturbed stochastic block model is characterized by four parameters: the two probabilities of intra-inter connection of  communities (denoted respectively by $p_1$ and $p_2 >0$) and two connectivity parameters $\kappa,\gamma$, chosen so that $\max (p_1, p_2)+\kappa \leq 1$:

-In the usual stochastic block model, vertices $i$ and $j$ are connected with probability $r_{i,j}$ where  
 $$ r_{ij}=\begin{cases}p_1&\mathrm{\ if\ }X_i, X_j \mathrm{\ belong\ to\ the\ same\ community\ }\cr
p_2& \mathrm{otherwise}\end{cases},$$
where $p_1$ and $p_2$ are in the same order  (the ratio $p_1/p_2$ is uniformly bounded).

-The geometric perturbation of the stochastic block model we consider is defined as follows. Conditionally on the values of  $X_i$,  the entries of the adjacency matrix $A=\left ( A_{ij}\right)$ are independent (up to symmetry) Bernoulli random variables   with parameter 
$q_{ij}=\kappa e^{-\gamma |X_i-X_j|^2}+ r_{ij}$.\\
 We remind that the motivation is independent to incorporate the fact that members from two different communities can actually be ``closer" in the latent space than members of the same community.\\
Thus in comparison with preceding model, the matrix $P$ of the geometric graph is now replaced with $Q:= \kappa P+ \begin{pmatrix} p_1J &p_2 J\cr p_2 J & p_1 J
\end{pmatrix},$ where we assume, without loss of generality, that $X_i, i\leq N/2$ (resp. $i \geq N/2+1$) belong to the same community.
The matrix $$P_0:=\begin{pmatrix} p_1J &p_2 J\cr p_2 J & p_1 J\end{pmatrix}$$ has two non zero eigenvalues which are 
$\lambda_1=N(p_1+p_2)/2$ with associated normalized eigenvector $v_1=\frac{1}{\sqrt N}(1, 1, \ldots 1)^\top$ and $\lambda_2=N(p_1-p_2)/2$ associated to $v_2=\sigma=\frac{1}{\sqrt N}(1,\ldots, 1, -1, \ldots -1)^\top$. Thus, in principle, communities can be detected from the eigenvectors of $P_0$ by using the fact that two vertices $i,j$ such that $v_2(i)v_2(j)=1$ belong to the same community. Our method can be generalized (using sign vectors) to more complicated models where the two communities are of different size, as well as to the case of $k$ communities (and thus the matrix $P_0$ has $k$ non zero eigenvalues).

\medskip

For the sake of notations, we write the adjacency matrix of the graph as :
$$A= P_0+  P_1 + A_c,$$
where $P_1=\kappa P$ with $P$ the $N\times N$-random symmetric matrix with entries $(P_{ij})$ -- studied in the previous section -- and $A_c$ is, conditionnally on the $X_i$'s a random matrix with independent Bernoulli entries which are centered. 

\subsection{Separation of eigenvalues: the easy case}
We are going to use spectral methods to identify communities. We therefore study in this section a regime where the eigenvalues of $A$ are well separated and the second eigenvector is approximately $v_2$, i.e. the vector which identifies precisely the two communities.

\bp \label{Prop:trivial} Assume that $$N (p_1-p_2)\gg  \sqrt N + \frac{N}{\gamma}.$$ Then, with probability tending to $1$, the two largest eigenvalues of $A$ denoted by $\rho_1\geq \rho_2$ are given by
$$\rho_i=\lambda_i(1+\smallO(1)), \: i=1,2.$$
Furthermore, with probability tending to 1, associated normalized eigenvectors (with non negative first coordinate) denoted by $w_1$ and $w_2$ satisfy  $\langle v_i, w_i \rangle=1-\smallO(1); \: i=1,2$.\ep

 Proposition \ref{Prop:trivial} implies that, in the  regime  considered, the spectral analysis of the adjacency matrix can be directly used to detect communities, in the same way it is a standard technique for the classical stochastic block model  (if $|p_1-p_2|$ is big enough compared to $p_1 + p_2$, which is the case here). 
 Finding the exact threshold $C_0$ such that if $N (p_1-p_2)= C_0( \sqrt N + \frac{N}{\gamma})$ then the conclusion of Proposition \ref{Prop:trivial} is still an open question. 

\subsection{Partial reconstruction when $\frac{N}{\gamma} \gg  \sqrt{N(p_1+p_2)}$}
From Theorem 2.7 in \cite{BenaychBordenaveKnowlesSR}, the spectral norm of $A_c$ cannot exceed $$\rho (A_c)\leq \left (\sqrt{ \kappa \frac{N}{\gamma}} +\sqrt{N(\frac{p_1+p_2}{2}+\mathcal{O}(\frac{\kappa}{2\gamma}))}\right )(1+\epsilon),$$
with probability tending to $1$, since  the maximal connectivity of a vertex does not exceed $N\big(\frac{p_1+p_2}{2} +  \frac{\kappa}{2\gamma}\big)(1+\smallO(1))$.  
In the specific regime where $$\frac{\kappa N}{2\gamma}\ll \sqrt{N \frac{p_1+p_2}{2}},$$  standard techniques \cite{bordenave2015} of communities detection would work, at the cost of  additional  perturbation arguments. As a consequence, we will concentrate on the reconstruction of communities  when $$\frac{\kappa N}{2\gamma}\gg \sqrt{N \frac{p_1+p_2}{2}}.$$ This essentially means that the spectrum of $A_c$ is blurred into that of $P_1.$
More precisely, we are from now going to consider the case where the noise induced by the latent random graph is of the same order of magnitude as the signal (which is the interesting regime):

\begin{equation}\tag{$H_2$}  \ \exists 0<c,C <1 \mathrm{\ s.t.\  } \lambda_2^{-1}\frac{\kappa N}{2\gamma} \in [c,C] , \frac{\lambda_2}{\lambda_1} \in [c,C]   \mathrm{\ and\ } \lambda_2 \gg \sqrt{\lambda_1} .\label{assH}
\end{equation}
If \eqref{assH} holds, then  the spectrum of $P_0+P_1$ overwhelms that of $A_c$. As a consequence, the problem becomes that of community detection based on  $P_0+P_1$, which will be done using spectral methods.

\medskip

To analyze the spectrum of $P_0+P_1$, we will use extensively the resolvent identity \cite{benaych2009eigenvalues}: consider $\theta \in \C\setminus \R$ and
set $S=P_0+P_1; R_S(\theta)=(S-\theta I)^{-1},$ $R_1(\theta):=(P_1-\theta I)^{-1}$. One then has that \be \label{eqresol}R_S(I+P_0 R_1)=R_1,\ee
where the variable $\theta$ is omitted for clarity when they are no possible confusion.
Since $P_0$ is a rank two matrix, then $P_0$ can be written as $P_0=\lambda_1v_1v_1^*+\lambda_2v_2v_2^*$ where $v_1$ and $v_2$ are the eigenvectors introduced before.

Eigenvalues of $S$ that are not eigenvalues of $P_1$ are roots of the rational equation $\det (I+P_0 R_1)=0$:
\begin{eqnarray}\label{EQ:Det1}
&\det(I+P_0 R_1)=&1+\lambda_1 \lambda_2\langle R_1v_1, v_1 \rangle \langle R_1 v_2, v_2 \rangle+\lambda_1\langle R_1v_1, v_1 \rangle \cr
&&+\lambda_2\langle R_1v_2, v_2 \rangle-\lambda_1\lambda_2\langle R_1v_1, v_2 \rangle^2.\end{eqnarray}
Let $\mu_1 \geq \mu_2 \geq \cdots \mu_N$ be the ordered eigenvalues of $P_1$ with associated normalized eigenvectors $w_1, w_2, \ldots, w_N$, then one has that $R_1(\theta)=\sum_{j=1}^N \frac{1}{\mu_j -\theta}w_jw_j^*.$
Denote, for  every  $j \in \{1, .., N\}$,  $r_j=\langle v_1, w_j\rangle$ and $s_j= \langle v_2, w_j\rangle$, so that Equation \eqref{EQ:Det1} rewrites into
\begin{align}\label{functg}
\notag\det(I+P_0 R_1(\theta))=:f_{\lambda_1,\lambda_2}(\theta)=&1+\sum_{j=1}^N \frac{1}{\mu_j-\theta} (\lambda_1 r_j^2+\lambda_2 s_j^2)\\
&+\lambda_1\lambda_2/2 \sum_{j\not=k}\frac{1}{(\mu_j -\theta)(\mu_k-\theta)}(r_j s_k-r_ks_j)^2.\end{align}
As mentioned before, we aim at using spectral methods to reconstruct communities based on the second eigenvector of $S$. As a consequence, these techniques may work only if (at least) two eigenvalues of $S$, that are roots of $\det(I+P_0 R_1(\theta))=0$ exit the support of  the spectrum of $P_1$, i.e., such that they are greater than $\mu_1$.

So we will examine  conditions under which there exist two real solutions to Equation~\eqref{functg},  with the restriction that they must be greater than  $\mu_1.$ If two such solutions exist,  by  considering the singularities in (\ref{eqresol}), then  two eigenvalues of $S$ indeed lie outside the spectrum of $P_1$.

\subsubsection{Separation of Eigenvalues in the rank two case.}
We now prove that two eigenvalues of $S$ exit the support of the spectrum of $P_1$. Recall the definition of the function $f_{\lambda_1,\lambda_2}$ given in Equation \eqref{functg} (or equivalently  Equation \eqref{EQ:Det1}).
One has that $\lim_{\theta \to \infty}f_{\lambda_1,\lambda_2}(\theta)=1$ , $f_{\lambda_1,\lambda_2}(\theta (\lambda_1))<0$ and similarly $f_{\lambda_1,\lambda_2}(\theta(\lambda_2))<0$, where $\theta(\cdot)$ is the function introduced in the rank 1 case. Thus two eigenvalues  exit the spectrum of $P_1$ if  $$\lim_{\theta \to \mu_1^+} f_{\lambda_1,\lambda_2}(\theta)>0.$$
\vspace{-0.2cm}
First, let us make the following claim (a consequence of   \eqref{H0} and  \eqref{assH}, see Lemma \ref{lemme:r_1}).

\begin{equation}\label{hypo2}\tag{$H_3$} \: \liminf_{N \to \infty} \lambda_1 r_1^2 >0.\end{equation}

\bl \label{lemm:sep}Assume \eqref{H0}, \eqref{assH} and \eqref{hypo2} hold and that there exists $\epsilon>0$ such that $$\lambda_2\geq 4 \mu_1(1+\epsilon)=4\kappa\frac{N}{2\gamma}(1+\epsilon).$$Then at least two eigenvalues of $P_0+P_1$  separate from the spectrum of $P_1$.\el 
\begin{proof}
Let us first assume that 
$$\mu_1 \mathrm{\ is\ isolated;\ there\ exists\ } \eta >0 \mathrm{\ such\ that\ for\ } N \mathrm{\ large\ enough\ } \mu_1>\mu_2+\eta.$$
In this case, we look at the leading terms in the expansion of $g$ as $\theta$ approaches $\mu_1$. 
It holds that
$$f_{\lambda_1,\lambda_2}(\theta)\sim \frac{1}{\theta-\mu_1} \left ( \lambda_1 \lambda_2 \sum_{j\geq 2}\frac{1}{\theta-\mu_j}(r_1s_j-r_js_1)^2-\lambda_1r_1^2-\lambda_2s_1^2\right ) .$$
Using that the spectral radius of $P_1$ does not exceed $\mu_1$, we deduce that 
\begin{align*}f_{\lambda_1,\lambda_2}(\theta)&\geq  \frac{1}{\theta-\mu_1}\left ( \frac{ \lambda_1\lambda_2}{2\theta}\sum_{j\geq 2}(r_1s_j-r_js_1)^2-\lambda_1r_1^2-\lambda_2s_1^2\right ) \cr
&\geq \frac{1}{\theta-\mu_1} \left ( \frac{\lambda_1 \lambda_2}{2 \theta} (r_1^2+s_1^2) -\lambda_1 r_1^2 -\lambda_2s_1^2\right)\geq \frac{1}{\theta-\mu_1}\lambda_1 (r_1^2+s_1^2)\epsilon,
\end{align*}
provided $\lambda_2\geq 2\mu_1(1+\epsilon).$ Note that if $\mu_1$ is isolated, the bound on $\lambda_2$ is improved by a factor of $2$.
\paragraph{}Now we examine the case where $\mu_1$ is not isolated. We then define
 $$I^*:=\{i : \: \limsup_{N \to \infty} \mu_i-\mu_1=0\},$$ and we define $\tilde v_i=\sum_{j \in I^*} \langle v_i, w_j \rangle w_j, $ $i=1,2.$
 Then mimicking the above computations, we get
 \begin{equation}\label{EQ:separation}f_{\lambda_1,\lambda_2}(\theta)\geq \frac{1+\smallO(1)}{\theta-\mu_1}\left ( \frac{\lambda_1 \lambda_2}{4\theta} (||\tilde v_1^2||+||\tilde v_2^2||)-\lambda_1 ||\tilde v_1^2||-
 \lambda_2||\tilde v_2^2||\right) \end{equation} so that  two eigenvalues separate from the rest of the spectrum as soon as $\lambda_2 >4\mu_1 (1+\epsilon)$. To get that statement we simply modify step by step the above arguments. This finishes the proof of Lemma \ref{lemm:sep} as soon as $\liminf_{N \to \infty} \lambda_1 r_1^2 >0$.
\end{proof}

The threshold exhibited for the critical value of $\lambda_2$  might not be the optimal one, however it is in the correct scale as we do not a priori expect a separation if $\lambda_2\leq \mu_1.$

\subsubsection{Partial reconstruction when $N\frac{p_1+p_2}{2}$ is known}
In the specific case where $N\frac{p_1+p_2}{2}$ is known beforehand for some reason, it is possible to weakly recover communities using Davis-Kahan $\sin(\theta)$-theorem under the same condition than Lemma  \ref{lemm:sep}.

We recall that this theorem states that if $M = \alpha xx^\top$ and $\widetilde{M} = \beta \widetilde{x} \widetilde{x}^\top$ is the best rank-1 approximation of $M'$, where both $x$ and $\widetilde{x}$ are normalized to $\|x\|=\|\widetilde{x}\|=1$, then
$$
\min\big\{ \|x - \widetilde{x}\|, \|x + \widetilde{x}\|\big\} \leq \frac{2\sqrt{2}}{\max\{|\alpha|,|\beta|\}}\|M-M'\|.
$$

\begin{theorem}\label{TH:known}Assume that \eqref{H0} and \eqref{assH} hold and that there exists $\epsilon>0$ such that 
$$\lambda_2\geq 4 \mu_1(1+\epsilon) \iff  \frac{p_1-p_2}{2}\geq \frac{2\kappa}{\gamma}(1+\epsilon),$$ then weak recovery of the communities is possible.
\end{theorem}
\begin{proof} We are going to appeal to Davis-Kahan theorem with respect to $$M= P_0-N\frac{p_1+p_2}{2} v_1v_1^\top= N\frac{p_1-p_2}{2}v_2v_2^\top$$ and $$M'=A - N\frac{p_1+p_2}{2} v_1v_1^\top = P_0+P_1+A_c- N\frac{p_1+p_2}{2} v_1v_1^\top = P_1+A_c+ M $$
As a consequence, let us denote by $\widetilde{x}$ the first eigenvector of $M'$ of norm 1 so that 
$$
\frac{1}{N}d_H(v_2,\sign(\widetilde{x})) \leq \|v_2 - \widetilde{x}\|^2 \leq \frac{8}{\lambda_2^2}\|P_1+A_c\|^2 = \frac{8}{\lambda_2^2}\mu_1^2(1+\smallO(1))\ .
$$
Weak reconstruction is possible if  the l.h.s.\ is strictly smaller than $1/2$, hence if $\lambda_2 \geq 4 \mu_1(1+\varepsilon)$.
\end{proof}

It is quite interesting that weak recovery is possible in the same regime where two eigenvalues of $P_0+P_1$ separate from the spectrum of $P_1$. Yet the above computations  imply that in order to compute $\widetilde{x}$, it is necessary to know $\frac{p_1+p_2}{2}$ (at least up to some negligible terms). In the standard stochastic block model, when $\kappa =0$, this quantity can be efficiently estimated since the $\frac{N(N-1)}{2}$ edges are independently drawn with overall probability $\frac{p_1+p_2}{2}$. As a consequence, the average number of edges is a good estimate of  $\frac{p_1+p_2}{2}$ up to its   standard deviation. The latter is indeed negligible compared to $\frac{p_1+p_2}{2}$ as it is in the order of   $\frac{1}{N}\sqrt{\frac{p_1+p_2}{2}}$.

On the other hand, when $\kappa \neq 0$, such trivial estimates are no longer available; indeed, we recall that the probability of having an edge between $X_i$ and $X_j$ is equal to $\frac{p_1+p_2}{2}+\frac{\kappa}{1+4\gamma}$, where all those terms are unknown (and moreover,  activations  of edges are no longer independent).  We study in the following section, the case where $p_1+p_2$ is not known. First, we will prove that  Assumption \eqref{hypo2} is actually always satisfied (notice that it was actually not required for weak recovery). In a second step, we will prove that \textsl{soft} recovery is possible, where we recall that this means we can output a vector $x\in \mathbb{R}^N$ such that $\|x\|=1$ and $x^\top v_2$ does not converge to 0. Moreover, we also prove that weak (and exact) recovery is possible if the different parameters $p_1$, $p_2$ and $\frac{1}{\gamma}$ are sufficiently separated.

\subsubsection{The case of unknown $p_1+p_2$}
We now proceed  to show that  Assumption \eqref{hypo2} holds in the regime considered.
\bl 
\label{lemme:r_1} Under \eqref{H0} and \eqref{assH}, one has that 1) for some  constant $C>0$, $\gamma r_1^2\geq C$. 
 and 2)  for some $\epsilon>0$ small enough, $\lambda_1 r_1^2\geq \epsilon.$\el 
The first point of Lemma \ref{lemme:r_1} implies \eqref{hypo2} with an explicit rate if $\gamma \leq A N^{\frac{1}{2}}$ for some constant $A$. The second point proves this result in the general case.





\begin{theorem}\label{Th:mainSBM}
If \eqref{H0} and \eqref{assH} hold true and $\lambda_1 > \lambda_2+2\frac{\kappa}{2\gamma}$ then 
 the correlation $|w_2^\top v_2|$ is uniformly bounded away from 0 hence soft recovery is always possible. Moreover, if the ratio  $\lambda_2/\mu_1$ goes to infinity, then $|w_2^\top v_2|$ tends to 1, which gives weak (and even exact at the limit) recovery.
\end{theorem}
An (asymptotic) formula for the level of correlation is provided at the end of the proof.

 \section{Experiments}\label{SE:expe}
 
 \begin{figure}[t]
\begin{center}
\noindent \includegraphics[scale=0.35]{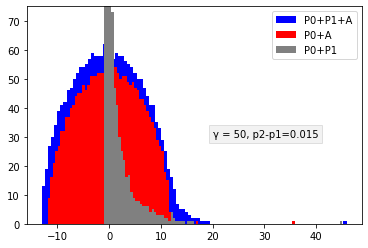} \includegraphics[scale=0.35]{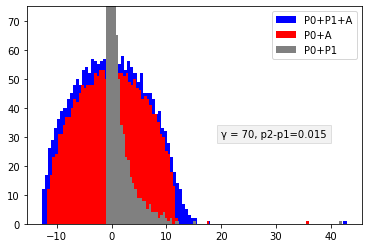}

\noindent \includegraphics[scale=0.35]{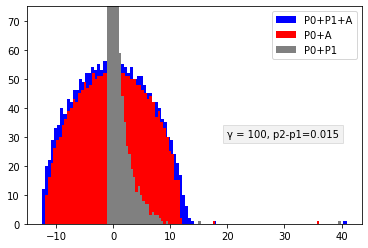} \includegraphics[scale=0.35]{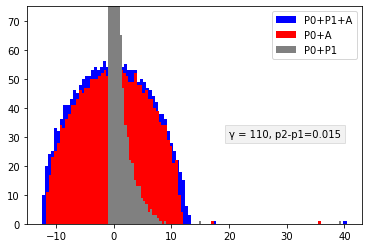}
\end{center}
\caption{The spectrum of the different block models for different values of $\gamma$.}
\label{Fig:Separation}
\end{figure}

The different results provided are theoretical and we proved that two eigenvalues separate from the bulk of the spectrum if the different parameters are big enough and sufficiently far from each other. And if they are too close to each other, it is also quite clear that spectral methods will not work. However, we highlight these statements in  Figure \ref{Fig:Separation}. It illustrates the effect of perturbation on the spectrum of the stochastic block models for the following specific values: $N=2000$, $p_1 = 2.5\%$, $p_2=1\%$, $\kappa =0.97$ and $\gamma \in \{50, 70, 100, 110\}$. Notice that for those specific values with get $\lambda_1=35$, $\lambda_2 = 15$ and $\mu_1 \in \{20,14.3, 10, 9.1\}$; in particular, two eigenvalues are well separated in the unperturbed stochastic block model.

The spectrum of the classical stochastic block model is coloured in red while the spectrum of the
perturbed one is in blue ( the spectrum of the conditionnal adjacency matrix, given the $X_i$'s is in gray). 
As expected, for the value of $\gamma=50$, the highest eigenvalue of $P_1$ is bigger than $\lambda_2$ and the spectrum of the expected adjacency matrix (in red) as some "tail". This prevents the separation of eigenvalues in the perturbed stochastic block model. Separation of eigenvalues starts to happen, empirically and for those range of parameters, around $\gamma=70$ for which $ \sqrt{\lambda_1} \leq \mu_1 =10 \leq \lambda_2$. \begin{figure}
\centering 
\includegraphics[scale=0.5]{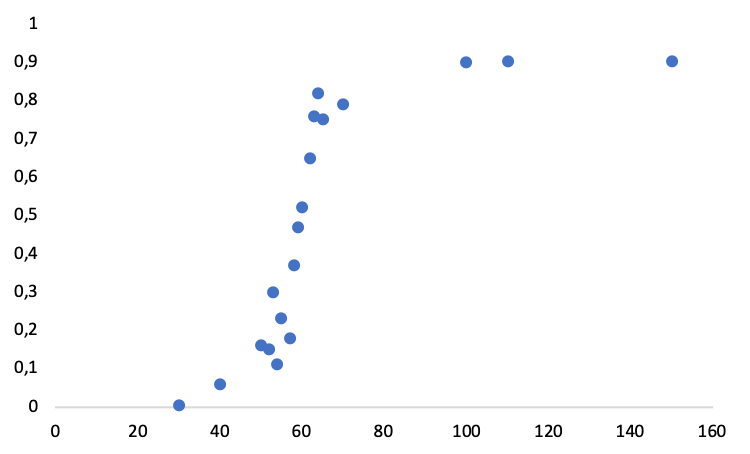}
\caption{The correlation between the second highest eigenvector and the community vector goes from 0 to 0.9 around the critical value $\gamma =60$.} \label{FIG: Corel}
\end{figure}

We also provide how the correlations between the second highest eigenvector and $\sigma$, the normalized vector indicating to which community vertices belong, evolve with respect to $\gamma$ for this choice of parameters, see Figure \ref{FIG: Corel}.

\section*{Conclusion}
The method exposed hereabove can be generalized easily. In the case where there are $k\geq 2$ communities of different sizes, $P_0$ has rank $k$. If $k$ eigenvalues of $S$ exit the support of the spectrum of $P_1$, then communities may be reconstructed using a set of $k$ associated (sign) eigenvectors, whether the parameters are known or not.

We have proved that spectral methods to recover communities are robust to slight mis-specifications of the model, i.e., the presence of endogenous noise not assumed by the model (especially when $p_1+p_2$ is not known in advance). Our results hold in the regime where $\frac{1}{\gamma} \gg \frac{\log N}{N}$ and with 2 communities (balancedness and the small dimension of latent variables were just assumed for the sake of computations) - those theoretical results are validated empirically by some simulations provided in the Appendix. 
Obtaining the same robustness results for more than 2 communities, for different types of perturbations and especially in the sparse regime $\frac{1}{\gamma}\sim p_i \sim \frac{1}{N}$ seems quite challenging as standard spectral techniques in this regime involve the non-backtracking matrix \cite{bordenave2015}, and its concentration properties are quite challenging to establish.

\section*{Acknowledgment} 
This research was supported by the Institut Universitaire de France. It was also supported in part by a public grant as part of the Investissement d'avenir project, reference ANR-11-LABX-0056-LMH, LabEx LMH, in a joint call with Gaspard Monge Program for optimization, operations research and their interactions with data sciences and by the French Agence Nationale de la Recherche under the grant number ANR19-CE23-0026-04.

\bibliographystyle{plain}
\bibliography{Graphs.bib}

\newpage

\begin{appendix}
\section{Additional  illustrating experiments}\label{App:simul}

In this section, we provide additional experiments that were run in the same conditions as those in Section \ref{SE:expe} on the difference that $p_2$ was set to the value $p_2=1.5\%$ (so that $p_2-p_1$ is quite small). Results are plotted on Figure \ref{Fig:Separation2} and, as expected, the second eigenvalue does not separate from the bulk. 

On the contrary, if we set $p_2 =4\%$ so that $p_2-p_1$ is large, then the second eigenvalue does separate from the bulk, even for small values of $\gamma$, see Figure \ref{Fig:Separation3}.

 \begin{figure}[h]
\begin{center}
\noindent \includegraphics[scale=0.3]{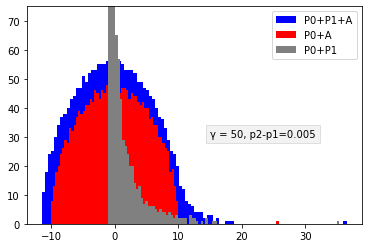} \includegraphics[scale=0.3]{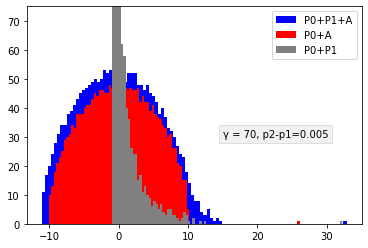}

\noindent \includegraphics[scale=0.3]{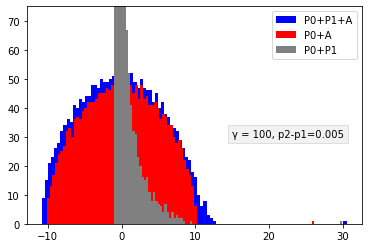} \includegraphics[scale=0.3]{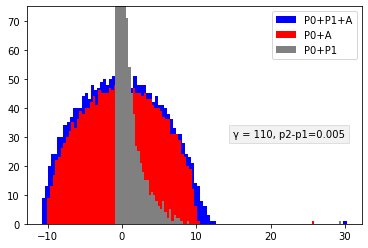}
\end{center}
\caption{The spectrum of the different block models for different values of $\gamma$ when $p_2-p_1 = 0.5\%$ hence the second eigenvalue does not separate from the bulk even for large values of $\gamma$.}
\label{Fig:Separation2}
\end{figure}

 \begin{figure}[h!]
\begin{center}
\noindent \includegraphics[scale=0.3]{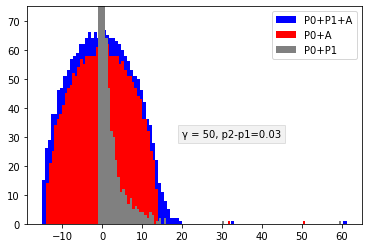} \includegraphics[scale=0.3]{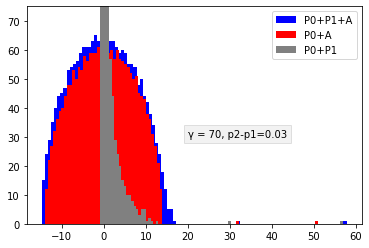}

\noindent \includegraphics[scale=0.3]{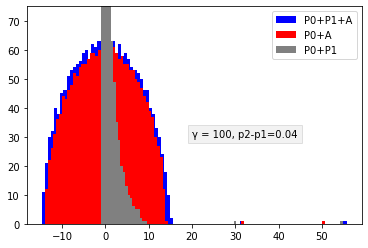} \includegraphics[scale=0.3]{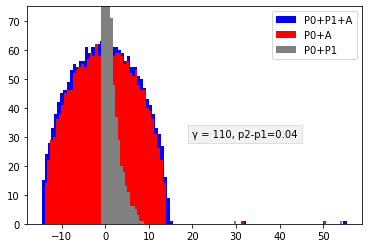}
\end{center}
\caption{The spectrum of the different block models for different values of $\gamma$ when $p_2-p_1 = 0.5\%$. The second eigenvalue separates from the bulk even for small values of $\gamma$}
\label{Fig:Separation3}
\end{figure}

\section{Additional results  and technical proofs of Section \ref{SE:spectral}}\label{App:spectral}
In this section, we gather additional results on the random graphs $P$, namely when it is connected (i.e., without isolated vertices) and whether it is possible to prove that some eigenvalues separate from the spectrum or not.

Then we will proceed to prove technical statements made in Section \ref{SE:spectral}.

\subsection{The connectivity regime}
Let us first consider a preliminary remark on the connectivity of the random graph. This result is for illustration purpose, as the connectivity (or not) of the geometric graphs would have no real impact on our main result, so we do not put too much emphasis on the exact threshold of connectivity.  On the other hand,  the result of Lemma \ref{Lemma:Connect} is rather intuitive as with very high probability, one of the $\|X_i\|^2$ are going to be in the order of $2\log(N)$, which indicate that the transition between connectivity or not should indeed be around $\log(N)/\log\log(N)$.

\begin{lemma} \label{Lemma:Connect} Assume that $ \frac{\log(N)}{\gamma \log\log N}\to \infty$ as $N\to \infty$. Then one has that 
$$\mbP (\exists \mathrm{\ an\ isolated\ vertex\ } i, 1\leq i\leq N ) \to 0 \mathrm{\ as\ } N\to \infty.$$ \label{prop:iso}
\end{lemma}
\begin{proof}Fix a vertex $i$. Conditionally on the $X_j$'s, the probability that $i$ is isolated is 
$$\prod_{j\not= i}(1-e^{-\gamma |X_i-X_j|^2}),$$
which we will integrate w.r.t.\ the distribution of independent $X_j$'s, $j \not=i$. Precisely, we get that the probability that there is an isolated vertex is upper-bounded by
\begin{align*}
\mbE \sum_i \prod_{j\not= i} (1-e^{-\gamma |X_i-X_j|^2}) &= N\mbE \Big(1- \frac{1}{1+2\gamma}e^{-\frac{2\gamma}{1+2\gamma}\frac{|X_i|^2}{2}}\Big)^{N-1}\\
&\leq N\Big(1- \frac{1}{1+2\gamma}e^{-\frac{2\gamma}{1+2\gamma}\frac{A^2}{2}}\Big)^{N-1} + Ne^{\frac{-A^2}{2}}
\end{align*}
for every $A>0$. In particular, the choice of $ne^{\frac{-A^2}{2}}=1/\log(N)$ gives that the probability of having an isolated vertex is smaller than
$$
N\exp\Big(-\frac{N-1}{N\log(N)}\frac{1}{1+2\gamma}(N\log(N))^{\frac{1}{1+2\gamma}}\Big)+\frac{1}{\log(N)}
$$

So as soon as $ \frac{\log(N)}{\gamma \log\log N}\to \infty$, the probability of having one isolated vertex goes to 0.\end{proof}

\subsection{Separation of eigenvalues}

\paragraph{}
We now examine the possibility that some eigenvalues of $P$ separate from the rest of the spectrum, as it could interfere with standard spectral methods used in community detection. For that purpose, we are going to study the moments of the spectral measure of $P$.

\bp \label{Prop2} Let $l\geq 2$ be a given integer, then the following holds:
\begin{align*}\lim_{N\to \infty} \displaystyle{\frac{1}{\gamma} \mbE \Tr \left (\frac{2\gamma P}{N}\right )^l }&= \frac{1}{l^2}\cr
\mathrm{Var}\displaystyle{\frac{1}{\gamma} \Tr \left (\frac{\gamma P}{N}\right )^l} &=\mathcal{O}\left (\frac{1}{N} \right )
\end{align*}
\ep 

 Proposition \ref{Prop2} implies in particular that the non-normalized spectral measure $$\mu(P)=\sum_{i=1}^N \delta_{\mu_i}$$ has asymptotically some positive mass on large values in the order of $\frac{N}{\gamma}$. This does not prevent that the largest eigenvalue separates from the others but it does not hold that the largest eigenvalue computed in Proposition \ref{Prop1} overwhelms the remaining eigenvalues.

Proposition \ref{Prop2} roughly states that  the largest eigenvalue does not macroscopically separate from the rest of the spectrum. Instead it is blurred into a cloud of large eigenvalues and thus cannot be distinguished. Notice that this phenomenon is rather different from the standard stochastic block model for which there exists a regime (in the average degree of the graph) where a finite number of eigenvalues really overwhelm the rest of the spectrum. 

\medskip

\begin{proof}
We use the fact that the $X_i$'s are Gaussian random variables to give an explicit formula for the moments of the spectral measure $\mu(P)$.
Let us use the standard method to derive its moments: let $l>1$ be given. One has that
\begin{eqnarray}
&& \E \sum_{i=1}^N \mu_i^l =\E \Tr P^l= \sum_{i_1, i_2, \ldots, i_l} \E \prod_{j=1}^l P_{i_j i_{j+1}} \label{eqmom},
\end{eqnarray}
using the convention that $i_{l+1}=i_1$.
Note that there may be some coincidences among the vertices $i_1, i_2, \ldots, i_l$ chosen in $\{1, \ldots, N\}.$ We forget for a while the precise labels of these vertices and denote them by $w_1,w_2, \ldots, w_l$ instead (keeping track of the coincidences however).

For each possible choice of the set of coincidences in (\ref{eqmom}), we denote by $k\geq 1$ the number of pairwise distinct indices (that we again label $w_1, w_2, \ldots w_k$).  We associate a graph $G_k$ on the vertices $\{w_1, w_2, \ldots w_k\}$ by simply drawing the edges $(w_j, w_{j+1}), j=1, \ldots, l$. Note that the graph may have multiple edges. It has no loops because $P_{ii}=0, $ for any  vertex $i$. 
Let $C_l$ denote the simple cycle with vertices $1, 2, \ldots, l$ in order. Then this graph corresponds to the case where there is no coincidence.  
When there are some coincidences, some vertices from $C_l$ are pairwise identified (excluding the possibility that subsequent vertices along the cycle are identified due to the fact that loops are not allowed).
For $k<l$ we denote by $\mathbf{G_k}$ the set of such graphs obtained by pairwise identifications of vertices from $C_l$ (excluding subsequent vertices). Note that
$\mathbf{G_l}=\{C_l\}$. 

Then one has that 
\begin{equation}
 \E \sum_{i=1}^N \mu_i^l = \sum_{k=2}^l \sum_{G_k \in \mathbf{G_k}} N (N-1)\cdots (N-k+1) \E \prod_{e \in G_k} P_e,\label{eq21}
\end{equation}
where in the above formula we have chosen the set of actual vertices among $\{1, \ldots, N\}$ and each edge $e\in G_k$ is repeated with its multiplicity in the product.
By standard Gaussian integration, using that $P_{(ij)}=\exp\{ -\gamma ||X_i-X_j ||^2\}$, one can easily check that 

\be \E \prod_{e \in G_k} P_e =\left ( \det (I +2\gamma L_{G_k})\right )^{-1},\label{eq22}
\ee

where $L_{G_k}$ is the Laplacian of $G_k$: we recall that the Laplacian of a graph $G=(V,E)$, $V=\{1, \ldots , k\}$ is the $k\times k$ matrix whose entries are
$$L_{ii}=-\mathrm{deg}(i), i=1,2, \ldots, k; L_{ij}= m_{ij}, i<j,$$
where $m_{ij}$ is the multiplicity of the non oriented edge $(i,j)$.

We now perform the expansion of $\det \left ( I +2\gamma L_{G_k}\right )$ according to the powers of $\gamma$.
By the matrix tree theorem (see \cite{chaiken1978matrix} e.g.), one has that 
\be \label{eq23}\det \left (  I +2\gamma L_{G_k}\right )=(2\gamma)^{k-1} k \times \sharp \{\mathrm{spanning\ trees\ of\ } G_k\}+ \sum_{i=2}^k (2\gamma)^{k-i} a_{k,i},\ee
for some coefficients $a_{k,i}$ which can be easily deduced from some minors of $L_{G_k}$.
Combining now equations (\ref{eq21}), (\ref{eq22}), (\ref{eq23}), and using that $C_l$ has $l$ spanning trees, we deduce that 

\begin{eqnarray}&\E \sum_{i=1}^N \mu_i^l &= N^l(1+\smallO(1)) \frac{1}{(2\gamma)^{l-1} l^2 (1+\smallO(\gamma^{-1})}\cr
&& + \sum_{k=2}^{l-1} N^k(1+\smallO(k^2/N)) \frac{1}{(2\gamma)^{k-1} c_k(1+\smallO(\gamma^{-1})}\cr
&&= \frac{N^l}{(2\gamma)^{l-1} l^2} \left (1+\mathcal{O}(\gamma^{-1})+ \mathcal{O}\left (\frac{\gamma}{N}\right )\right).\label{defck}
\end{eqnarray}
In the second line of (\ref{defck}), the constant $c_k$ is given by $$c_k^{-1}=\sum_{G_k \in \mathbf{G_k}}\frac{1}{k \sharp \{\mathrm{spanning\ trees\ of\ } G_k\}}.$$
Thus we have proved the first statement of Proposition \ref{Prop2}.

Let us now turn to the variance :
$$\mathrm{Var} (\Tr P^l)=\E \left ( \Tr P^l \Tr P^l\right )-\left (\E \Tr P^l\right )^2.$$

We again developp the product 
$$ \Tr P^l \Tr P^l=\sum_{i_1, i_2, \ldots, i_l}  \prod_{k=1}^l P_{i_j i_{j+1}} \sum_{i_1', i_2', \ldots, i_l'}  \prod_{k=1}^l P_{i_j' i_{j+1}'}$$
and draw the associated graphs (forgetting the labels) on possibly $2l$ vertices.
If the two graphs are disconnected (this means that the two sets  $\{i_1, i_2, \ldots, i_l\}$ and $\{i_1', i_2', \ldots, i_l'\}$ are disjoint, then the expectation of the product splits by independance. The combined contribution of each subgraph to the variance will thus be in the order of $l^2/N$ times $\left (\E \Tr P^l\right )^2$. This comes from the fact that one has to choose $2k$ pairwise distinct indices when combining the two graphs (while twice $k$ pairwise distinct indices when considering the squared expectation of the Trace). Thus, by definition of the variance, the only graphs which are contributing to the variance are those for which at least one vertex from $\{i_1, i_2, \ldots, i_l\}$ and $\{i_1', i_2', \ldots, i_l'\}$ coincide. This means that using the same procedure as above, one can restrict to the set of graphs $G_k, k\leq 2l-1$ which are obtained from $C_{2l}$ by at least one identification.

From the above it is not difficult to check that $\mathrm{Var} \frac{1}{\gamma}\Tr \left (\frac{\gamma P}{N}\right )^l =\mathcal{O}\left (\frac{1}{N} \right ).$
This finishes the proof of Proposition \ref{Prop2}.\end{proof}

\subsection{Proof of Proposition \ref{lemme:gammaconstant}}

We first show that there exists a constant $C_1$ such that $$\frac{\mu_1(P)}{N C_1(\gamma_0)}\geq 1$$ for $N$ large enough.
For $i=1, \ldots, N$ we set $d(i):=\sum_j P_{ij},$ which we call "the connectivity" of $i$. By the Perron Frobenius theorem the largest eigenvalue of $P$ cannot exceed the maximal connectivity of a vertex, (which can be proved to be strictly greater than $\frac{N}{1+4\gamma}$). However the number of vertices whose connectivity is such high is negligible with respect to $N$ (it is not obvious such a number grows to infinity actually). 
Because all the entries of $P$ are positive, one knows that the largest eigenvalue of $P$ is simple and is equal to the spectral radius of $P$. Furthermore, one has that $$\mu_1(P)= \lim_{l\to \infty} \frac{ \langle v_1,P^l v_1\rangle}{ \langle v_1,P^{l-1} v_1\rangle}  $$ where $\sqrt{N}v_1=\tilde v_1=(1, 1, \ldots, 1)^t$.
Actually we are going to show that $$\mu_1(P)^2=(1+\smallO(1))  \frac{ \langle v_1,P^{2l+2} v_1\rangle}{ \langle v_1,P^{2l} v_1\rangle} \mathrm{\ for\ } l=\ln N. $$ 

First one has that $$\mu_1(P)^2\geq  \frac{ \langle v_1,P^{2l+2} v_1\rangle}{ \langle v_1,P^{2l} v_1\rangle} \mathrm{
 for\ } l=\ln N. $$ 
Now we show some concentration estimates for both the numerator and denominator, for $l\sim \ln N$ showing that to the leading order they concentrate around their mean which is enough to show that $$\mu_1\geq C_1(\gamma) N(1+\smallO(1)).$$
Observe that $\langle \tilde v_1, P^l \tilde v_1 \rangle= \sum_{i,j,i_1, \ldots, i_{l-1}} P_{ii_1}P_{i_1i_2}P_{i_{l-1}j}$ is a sum of at most $N^{l+1}$ terms. Each of the summands if a function of the Gaussian vector $X=(X_1, X_2, \ldots, X_N)^t.$ We are going to show that $X\mapsto \sum_{i,j,i_1, \ldots, i_{l-1}} P_{ii_1}P_{i_1i_2}P_{i_{l-1}j}$ is Lipschitz with Lipschitz constant in the order of $N^{(2l+1)/2}$ for some constant $C$ large enough. 
As $\mathbb{E}\sum_{i,j,i_1, \ldots, i_{l-1}} P_{ii_1}P_{i_1i_2}P_{i_{l-1}j}=(N C_2(\gamma_0))^{l+1}(1+\smallO(1))$ for some constant $C_2 (\gamma_0) >0$, this will be enough to ensure using standard concentration arguments for Gaussian vectors that
$$\mathbb{P}\left ( |\sum_{i,j,i_1, \ldots, i_{l-1}} P_{ii_1}P_{i_1i_2}P_{i_{l-1}j}-(C_2(\gamma_0)N)^{l+1}|\geq A N ^{(2l+1)/2}\right )\leq 2e^{-2A^2}.$$
Thus this implies that a.s. $$\lim_{N \to \infty}\frac{\sum_{i,j,i_1, \ldots, i_{l-1}} P_{ii_1}P_{i_1i_2}P_{i_{l-1}j}}{(C_2(\gamma_0)N)^{l+1}}=1.$$
Consider two vectors $X$ and $Y$. One has that 
\begin{eqnarray}
&&\Big |\sum_{i,j,i_1, \ldots, i_{l-1}} P_{ii_1}P_{i_1i_2}P_{i_{l-1}j}(X)-\sum_{i,j,i_1, \ldots, i_{l-1}} P_{ii_1}P_{i_1i_2}P_{i_{l-1}j}(Y)\Big |\cr
&&\leq \sum_{k=0}^{l-1}\sum_{i,j,i_1, \ldots, i_{l-1}} P_{ii_1}(X)P_{i_1i_2}(X)\Big |P_{i_ki_{k+1}}(X)-P_{i_ki_{k+1}}(Y)\Big |P_{i_{k+1}i_{k+2}}(Y) \ldots P_{i_{l-1}j}(Y)\cr
&&\leq \alpha \sum_{k=0}^{l-1}\sum_{i,j,i_1, \ldots, i_{l-1}} \prod_{l=0}^{k-1}P_{i_li_{l+1}}(X)\prod_{l=k+1}^{l-1}P_{i_li_{l+1}}(Y)\Big | |X_{i_k}-X_{i_{k+1}}|-|Y_{i_{k}}-Y_{i_{k+1}}|\Big |, \cr &&\label{fin}
\end{eqnarray}
where in the last line we have used the fact that $x\mapsto e^{-\gamma x^2}$ is $\alpha$-Lipschitz. 
The constant $\alpha$ can be chosen as $\alpha = 4 \sqrt \gamma \sup_x |xe^{-x^2}|.$
Consider the sum in (\ref{fin}). We note $\sum_*$ the sum over indices $i,j,i_1, \ldots, i_{l-1}$ and $k$ in the following. One has that 
\begin{align*} &\sum_{*} \prod_{l=0}^{k-1}P_{i_li_{l+1}}(X)\prod_{l=k+1}^{l-1}P_{i_li_{l+1}}(Y)\Big | |X_{i_k}-X_{i_{k+1}}|-|Y_{i_{k}}-Y_{i_{k+1}}|\Big |\cr
&\leq  \sqrt{ \sum_{*}\prod_{l=0}^{k-1}P_{i_li_{l+1}}^2(X)\prod_{l=k+1}^{l-1}P_{i_li_{l+1}}^2(Y)}\sqrt{\sum_{*}\Big | 
|X_{i_k}-X_{i_{k+1}}|-|Y_{i_{k}}-Y_{i_{k+1}}|\Big |^2}\cr
&\leq  N^{\frac{l+1}{2}} N^{\frac{l-1}{2}}\left ( \sum_k 8 |X-X_kv_1 -(Y-Y_k v_1)|^2\right)^{\frac{1}{2}}\cr
&\leq CN^{(2l+1)/2} ||X-Y||.
\end{align*}

We now show that $$\mu_1(P)^2\leq (1+\smallO(1)) \frac{ \langle v_1,P^{2l+2} v_1\rangle}{ \langle v_1,P^{2l} v_1\rangle} \mathrm{\ for\ } l=\ln N. $$ Denote by $w_i, i=1,\ldots, N$ a set of orthonormalized eigenvectors of $P$.
Equivalently the above means that $$\sum_{i>1}\mu_i^{2l}(\mu_1^2-\mu_i^2) \langle w_i, v_1\rangle^2=\smallO(1)\sum_{i\geq 1}\mu_i^{2l+2} \langle w_i, v_1\rangle^2.$$ 
Fix $\epsilon >0$. Set $r^2:= \sum_{i: \mu_1-|\mu_i|<\epsilon }\langle w_i, v_1\rangle^2.$ The first sum in the above then does not exceed:
$$2\epsilon r^2 \mu_1^{2l+1}+\mu_1^{2l+2}(1-r^2)(1-\epsilon)^{2l}.$$ This is $\smallO(1) \mu_1^{2l+2}r^2$ provided that $r^2 \geq \eta$ for some $\eta>0$. This is the fact we prove below. To that aim we show that $\langle w_1, v_1 \rangle^2 \geq \eta.$ Using that $w_1$ (associated to $\mu_1$) has non negative coordinates and is normalized to $1,$ one has that  $\langle w_1, v_1\rangle \geq \frac{1}{\sqrt N |w_1|_{\infty}}.$ Thus it is enough to show that $\limsup \sqrt N |w_1|_{\infty} <\infty.$   Assume this is not the case : then there exists a sequence $A_N\to \infty$ such that $\sqrt N |w_1|_{\infty}\geq A_N$ (along some subsequence). In particular let $w_{i_0}=\max w_i \geq \frac{A_N}{\sqrt N}.$
Fix $\delta>0$ small. Set $J:=\{j, w_j \geq \delta w_{i_0}\}.$ Then one has that $\sharp J \leq \frac{N}{\delta^2 A_N^2} \ll N$.
Using this in the expression 
$$\mu_1=\sum_{j \in J}P_{i_0j} \frac{w_j}{w_{i_0}}+\sum_{j \notin J}P_{i_0j} \frac{w_j}{w_{i_0}} $$ one deduces that 
$$\mu_1 \leq N \delta +\sharp J,$$ which is a contradiction.
This finishes the proof of Proposition \ref{lemme:gammaconstant}.

\subsection{Proof of Lemma \ref{lem: 1}}
Let us first introduce  some notations and key results for the proof. The function $$\theta: r\geq 0 \mapsto \theta(i,r):=\int_{D(X_i, \sqrt r )} \frac{1}{2\pi} e^{-|x|^2/2}d\lambda_2(x),$$
where $D(X_i, \sqrt r)$ is the disk centered at $X_i$ of radius $\sqrt r$.

Notice that the following holds for all $r >0$
$$
e^{-\frac{\|X_i\|^2}{2}}\Big(1-e^{-\frac{r}{2}}\Big)e^{-2\|X_i\|\sqrt{r}} \leq \theta(i,r) \leq e^{-\frac{\|X_i\|^2}{2}}\Big(1-e^{-\frac{r}{2}}\Big)e^{2\|X_i\|\sqrt{r}}.
$$
It also holds that
$$
2e^{-\|X_i\|^2}(1-e^{-r})\leq \theta(i,r) \leq e^{-\frac{\|X_i\|^2}{4}}(e^{\frac{r}{2}}-1)\
$$

and moreover if $r_1 > r_0$ then we immediately have
$$\theta(i,r_1)-\theta(i,r_0)\leq \frac{r_1-r_0}{2}. $$

Conditionally on $X_i$, the number of vectors among the $X_j'$s whose distance to $X_i$ falls in the interval $I$ is a binomial random variable $\mathrm{Bin}(N-1, \theta (i,l(I))).$ So we recall  the following basic concentration argument (see equivalently Theorem 2.6.2 in \cite{vershynin2018high}). Let $Z$ be a binomial random variable with distribution $\mathrm{Bin}(m,p)$. There exists a constant $\alpha>0$ ( if $p<4/5$, one can choose $\alpha=1/32$) such that for any $C>0$, one has 
$$\mbP\left  (|Z-mp| \geq C \sqrt{m p}\right) \leq 2e^{-\alpha C^2}.$$  

\medskip

We can now turn to the proof of Lemma \ref{lem: 1} itself. Let $\varepsilon >0$ be fixed (its specific value is tuned at the end of the proof) and  $i \in [N]$ be a fixed index such that $|X_i|^2\leq \frac{2\ln \gamma}{\gamma}$. We are going to show that 
$$S:= \sum_{j=1}^N e^{-\gamma |X_i-X_j|^2}= c_0 \frac{N}{\gamma } \left (1\pm \smallO(1) \right),$$
where 
$$ c_0:= \lim_{N \to \infty}\frac{ \gamma}{N}\sum_{k=1}^{2\frac{\ln \gamma}{\varepsilon}}\mathbf{n_k^{(i)}}e^{-k\varepsilon},
$$
with $\forall k=1 , \ldots, 2\frac{\ln \gamma}{\varepsilon}$,
$$ \mathbf{n_k^{(i)}}:= N\Big(\theta  \big(i, \frac{(k+1)\varepsilon}{\gamma } \big)- \theta  \big(i,\frac{k\varepsilon}{\gamma }\big)\Big).$$
As $\gamma$ goes to infinity with $N$, then it holds that
$$
N e^{-\frac{\|X_i\|^2}{2}}\Big(\frac{\varepsilon}{2\gamma}-\mathcal{O}(\frac{\ln^2\gamma}{\gamma^2})\Big)\leq\mathbf{n_k^{(i)}} \leq N\frac{\varepsilon}{2\gamma}
$$
so that if $\|X_i\|^2 \leq 2 \frac{\ln \gamma}{\gamma}$ then $\mathbf{n_k^{(i)}} \simeq \frac{N\varepsilon}{2\gamma }$ which ensures that  $c_0=\frac{1}{2}(1+\smallO(1))$ is well-defined.

To control $S$, we split this sum into three parts, depending on the distances from $X_j$ to $X_i$, as follows 

 \begin{align*} S &= \underbrace{\sum_{j: d^2(X_i,X_j)\in [\frac{\varepsilon}{\gamma}, \frac{2\ln \gamma}{\gamma}]} e^{-\gamma |X_i-X_j|^2}}_{S_1} +  \underbrace{\sum_{j: d^2(X_i,X_j) < \frac{\varepsilon}{\gamma}} e^{-\gamma |X_i-X_j|^2}}_{S_2}\\
 &\hspace{1.5cm} +\underbrace{\sum_{j: d^2(X_i,X_j) >   \frac{2\ln \gamma}{\gamma}} e^{-\gamma |X_i-X_j|^2}}_{S_3} \end{align*}

We first focus on $S_1$ that we are  going to further decompose as a function of the distance from $X_j$ to $X_i$ : define for $k \in \{1,\ldots, 2\frac{\ln\gamma}{\varepsilon}\}$
$$n_k^{(i)}:=\sharp \Big \{l, d^2(X_l, X_i)\in \left [\frac{k\varepsilon }{\gamma},\frac{(k+1)\varepsilon}{\gamma}\right  [\Big \}\ .$$

Then one has 
\begin{align*}S_1 &\leq  \sum_{k=1}^{2\frac{\ln\gamma}{\varepsilon}} e^{- k\varepsilon}n_k^{(i)}\cr
&= \sum_{k=1}^{2\frac{\ln\gamma}{\varepsilon}} e^{- k\varepsilon} \mathbf{n_k^{(i)}} +\sum_{k=1}^{2\frac{\ln\gamma}{\varepsilon}}   e^{- k\varepsilon}\left (n_k^{(i)}-\mathbf{n_k^{(i)}}\right ) \cr
&= \frac{N}{2\gamma }\big(1+\smallO(1)\big) +\sum_{k=1}^{2\frac{\ln\gamma}{\varepsilon}}  e^{- k\varepsilon}(n_k^{(i)}-\mathbf{n_k^{(i)}})  ,
\end{align*}
where the last equality comes from the approximation of $\mathbf{n_k^{(i)}}$ as $N$ and $\gamma$ goes to infinity.  It also holds that 
\begin{align*}S_1 &\geq  \sum_{k=1}^{2\frac{\ln\gamma}{\varepsilon}} e^{- (k+1)\varepsilon}n_k^{(i)}\cr
&= \sum_{k=1}^{2\frac{\ln\gamma}{\varepsilon}} e^{- (k+1)\varepsilon} \mathbf{n_k^{(i)}} +\sum_{k=1}^{2\frac{\ln\gamma}{\varepsilon}}   e^{- (k+1)\varepsilon}\left (n_k^{(i)}-\mathbf{n_k^{(i)}}\right ) \cr
&\geq \frac{N}{2\gamma }\big(1- 2\varepsilon-\smallO(1)\big) +\sum_{k=1}^{2\frac{\ln\gamma}{\varepsilon}}  e^{- {(k+1)}\varepsilon}(n_k^{(i)}-\mathbf{n_k^{(i)}}).
\end{align*}

It remains to control the different errors $n_k^{(i)}-\mathbf{n_k^{(i)}}$. It holds that,
$$\mbP_{X_i} \left (\exists 1\leq k\leq \frac{2\ln \gamma}{\varepsilon}, \:  |n_k^{(i)}-\mathbf{n_k^{(i)}}|\geq \varepsilon \mathbf{n_k^{(i)}} \right)\leq 8\frac{\ln(\gamma)}{\varepsilon} e^{-\alpha \varepsilon^2 \frac{N}{4\gamma} },$$
because each $\mathbf{n_k^{(i)}} \simeq \frac{N\varepsilon}{2\gamma}$ as $\gamma$ increase to infinity with $N$. At the end, we obtained that for each $X_i$ such that $\|X_i\|^2 \leq \frac{2\log(\gamma)}{\gamma}$, then 
\be \left |S_1 -  \frac{N}{2\gamma}\right | \leq \frac{N}{2\gamma}\big(3\varepsilon + \smallO(1)\big) \quad \mathrm{\ with\ proba\ at\ least\ } 1- 8\frac{\ln(\gamma)}{\varepsilon} e^{-\alpha \varepsilon^3  \frac{N}{4\gamma}}\ .\label{a}\ee

\medskip

Let us  now focus on $S_2$ which is obviously smaller than $n_0^{(i)}$ where
$$n_0^{(i)}:=\sharp \{j, d^2(X_i, X_j)<\frac{\varepsilon}{\gamma}\}.$$
Moreover, because of the concentration of binomials, it holds that
$$\mbP_{X_i} \left (n_0^{(i)}\geq 2N\theta(i,\frac{\varepsilon}{\gamma })\right)\leq 2e^{-\alpha N\theta(i,\frac{\varepsilon}{\gamma }) }.
$$
Now as $\gamma$ goes to infinity with $N$, then for $\gamma$ large enough, the following holds $$\frac{\varepsilon}{4\gamma }\leq \theta(i,\frac{\varepsilon}{\gamma })\leq \frac{\varepsilon}{2\gamma }$$
which ensures that 
$$ \mbP_{X_i} \left (n_0^{(i)}\geq \frac{N\varepsilon}{\gamma}\right)\leq 2e^{-\alpha \frac{N\varepsilon}{4\gamma } }.$$
As a consequence we have shown that
\begin{equation}\label{b}S_2 \leq \frac{N\varepsilon}{\gamma } \quad \mathrm{\ with\ probability\ at\ least\ } 1-2e^{-\frac{\alpha}{4} \frac{N\varepsilon}{\gamma }} \ .\end{equation}
Last, by the very definition of $S_3$, it always holds that 
\be\label{c} S_3\leq Ne^{-\ln \gamma^2}\leq \frac{N}{\gamma^2} .\ee

Combining (\ref{a}), (\ref{b}) and (\ref{c}), we obtain that with probability at most  $$2e^{-\frac{\alpha}{4} \frac{N\varepsilon}{\gamma}}+8\frac{\ln(\gamma)}{\varepsilon} e^{-\alpha \varepsilon^3  \frac{N}{4\gamma }}$$
one has that
$$
\Big|S - \frac{N}{2\gamma}\Big| \leq \frac{N}{2\gamma}\big( 5\varepsilon + \smallO(1)\big) .
$$
As a consequence, as $N$ grows to infinity, one has
\begin{align*}
&\mbP \left ( \exists i: |X_i|^2 \leq\frac{2 \ln \gamma}{ \gamma} \mathrm{\ and\ } \Big|\sum_{j=1}^Ne^{-\gamma |X_i-X_j|^2} -\frac{N}{2\gamma}\Big|\geq \frac{N}{2\gamma}\big(5\varepsilon+\smallO(1)\big)  \right)\\
& \leq 4N\frac{\ln \gamma}{\gamma}\Big(e^{-\frac{\alpha}{4} \frac{N\varepsilon}{\gamma }}+4\frac{\ln(\gamma)}{\varepsilon} e^{-\alpha \varepsilon^3  \frac{N}{4\gamma }} \Big)\to 0
\end{align*}
by choosing $\varepsilon = \Big(\frac{N}{\gamma \ln \gamma }\Big)^{-1/4}$ (so  that $\varepsilon$ goes to 0 as intended) and because $\frac{N}{\gamma \ln \gamma}$ goes to infinity. This proves Lemma \ref{lem: 1}.

\subsection{Proof of Lemma  \ref{lemme:ub}}
The proof is almost identical to that of Lemma \ref{lem: 1}. The only difference is that we cannot approximate  $\mathbf{n_k^{(i)}}$ by $\frac{N\varepsilon}{2\gamma}$ because $e^{-\frac{\|X_i\|^2}{2}}$ might go to 0. Yet it still holds that   $\mathbf{n_k^{(i)}} \leq \frac{N\varepsilon}{2\gamma}$.
And thus, we can easily prove the weaker statement
\begin{align*}
&\mbP \left ( \exists i, \sum_{j=1}^Ne^{-\gamma |X_i-X_j|^2} \geq \frac{N}{2\gamma}\big(1+5\varepsilon+\smallO(1)\big)  \right)\\
& \leq 4N\Big(e^{-\frac{\alpha}{4} \frac{N\varepsilon}{\gamma }}+4\frac{\ln(\gamma)}{\varepsilon} e^{-\alpha \varepsilon^3  \frac{N}{4\gamma }} \Big)\to 0
\end{align*}
with the same choice of $\varepsilon$, assuming Assumption  (\ref{H0}) holds.
%

\subsection{Proof of Lemma \ref{lem: UB}}
If we  can show that for any $ i \in J$ and with probability close to 1, it holds that \be \label{reste}\sum_{j\notin J} P_{ij}\ll \frac{N}{\gamma} ,\ee
then the result would be a direct consequence of Lemma \ref{lem: 1}. 

 By the very definition of $J$, if $j\notin J$, then necessarily 
$|X_j|^2 \geq \frac{2\ln \gamma}{\gamma}.$ Notice that  $|X_j|^2 \geq (3+\epsilon) \frac{\ln \gamma}{\gamma}$ then $\gamma |X_i-X_j|^2\geq (1+\epsilon)\ln \gamma$
so that for any $i\in J,$ this immediately yields that 
 $$\sum_{j, |X_j|^2 \geq (3+\epsilon) \frac{\ln \gamma}{\gamma}}P_{ij}\leq \frac{N}{\gamma^{1+\epsilon}}\ll \frac{N}{\gamma}.$$ 
 This is enough to obtain (\ref{reste}) for the contribution of such indices. Note also that the same argument is valid to get (\ref{reste}) for the subsum
(keeping $i\in J$ fixed) $$\sum_{j,\gamma |X_i-X_j|^2\geq (1+\epsilon)\ln \gamma} P_{ij}\leq\frac{N}{\gamma^{1+\epsilon}}\ll \frac{N}{\gamma}.$$

Thus we only need to consider indices $i\in J$ and $j \notin J$ such that $\gamma |X_i-X_j|^2\leq (1+\epsilon)\ln \gamma$. This implies in particular that necessarily  
$\|X_j\|^2 \leq 8\frac{\ln(\gamma)}{\gamma}$. 
Consider therefore such an index $i$ and let 
$$S:= \sum_{j: \gamma |X_j-X_i|^2\leq (1+\epsilon)\ln \gamma, |X_j|^2\geq \frac{2\ln \gamma }{\gamma}} e^{-\gamma |X_i-X_j|^2}.$$
Because $\|X_j\|^2 \leq 8 \frac{\ln \gamma}{\gamma}$ then the number of  indices $j \not \in J$ is smaller than than $16 N \frac{\ln \gamma}{\gamma}$ with probability at least $1-e^{-\alpha 8N \frac{\ln \gamma}{\gamma}}$. As a consequence, the sum above is composed of at most $16 N \frac{\ln \gamma}{\gamma}$ terms, all smaller than 1. Obviously, if they are all smaller than $\frac{1}{\ln^2 \gamma}$ then $S \leq 16 \frac{N}{\gamma\ln \gamma }\ll  \frac{N}{\gamma} $.

So this implies that  it only remains to control the sum $S$ for indices $i \in J$ such that for some $j \not \in J$ it holds that $\|X_i-X_j\|^2 \leq \frac{4\ln\ln\gamma}{\gamma}$. This implies that such indices $i \in J$ must satisfy   $$ 2 \frac{\ln \gamma}{\gamma}\geq |X_i|^2\geq 2\frac{\ln \gamma}{\gamma}\left (1-2\sqrt{2\frac{\ln \ln \gamma}{\ln \gamma}}\right ) .$$
And, using the same argument as before, there are at most $8\frac{N}{\gamma} \sqrt{\ln \gamma \ln\ln \gamma}$  such indices with arbitrarily high probability (as $\gamma$ goes to infinity). On the other hand, $\sharp J$ (the cardinality of $J$)  is, with arbitrarily high probability, in the order of $N\frac{\ln \gamma}{\gamma}$

This gives a lower bound on the spectral radius of $P_J$: let $v$ be the unit vector $v=\frac{1}{\sqrt{ \sharp J}}(1, \ldots, 1)^t$ (of dimension $\sharp J$). Then \begin{align*}\langle P_J v, v \rangle &\geq  \frac{\sharp J-8\frac{N}{\gamma}\sqrt{\ln\gamma \ln\ln\gamma}}{\sharp J}\frac{N}{2\gamma}(1-\smallO(1))\\
&\geq \frac{N}{2\gamma}(1-\smallO(1))\left (1-8\sqrt{\frac{\ln\ln\gamma}{\ln\gamma}}\right )\\
& \geq \frac{N}{2\gamma}(1-\smallO(1))\ .
\end{align*}
Hence Lemma \ref{lem: UB} is proved.

\section{Technical proofs of Section \ref{SE:SBM}}\label{App:SBM}

\subsection{Proof of Proposition \ref{Prop:trivial}}
The preceding proof can be easily modified to obtain the following bounds on the spectral radii : there exist constants $c_0=1/2, C>0$ so that with high probability 
$$\rho(P_1)\leq c_0 \frac{N}{\gamma}; \: \rho (A_c)\leq C \sqrt{N}.$$
Following \cite{furedi1981eigenvalues}, we first prove that the largest eigenvalue of $A$ is up to a negligible error (in the appropriate regime of $p_1, p_2, \gamma$) that of $P_0$. More precisely, it holds with arbitrarily high probability that $$\langle Av_1, v_1 \rangle =N\frac{p_1+p_2}{2}+\mathcal{O}\left ( \frac{N}{\gamma}+\sqrt{N\left (\frac{p_1+p_2}{2}+\frac{\kappa}{2\gamma}\right )}\right ).$$ 
It easily follows that the largest eigenvalue $\rho_1(A)$ of $A$ satisfies 
$$\rho_1\geq \lambda_1 \Big(1+\mathcal{O}(\frac{1}{\gamma (p_1+p_2)}+\frac{1}{\sqrt{N}})\Big).$$ In addition decomposing a normalized eigenvector $v$ associated to $\rho_1$ as 
 $$v= r_1 v_1+r_2v_2+\sqrt{1-r^2}w$$ for some normalized vector $w$ orthogonal to $v_1$ and $v_2$ and where $r^2=r_1^2+r_2^2$, then one has that
 $$\langle Av, v\rangle = r_1^2 N\frac{p_1+p_2}{2} + \mathcal{O}( \sqrt N +\frac{N}{\gamma}) f(r) +N\frac{p_1-p_2}{2}r_2^2$$ 
 for some function $f(\dot)$ such that $\|f\|_\infty \leq 1$.
 Thus it follows that $r_1=1+\mathcal{O}(\frac{1}{\gamma}+N^{-\frac{1}{2}}).$ This finishes the proof that 
 the largest eigenvalue (and eigenvector) of $A$ and $P_0$ almost coincide. 
 Similarly, since $$\langle Av_2, v_2 \rangle =\lambda_2+\mathcal{O}\left ( \frac{N}{\gamma}+\sqrt N \right ),$$ the same arguments imply that 
 the second largest eigenvalue of $A$ and $P_0$ coincide provided $$N (p_1-p_2)\gg  \sqrt N + \frac{N}{\gamma}.$$
 And associated normalized eigenvectors  coincide asymptotically, following the same basic perturbation argument.

\subsection{Proof of Lemma \ref{lemme:r_1}}
We first prove the first point.  The objectif is to lower-bound  $\langle v_1, w_1\rangle.$ Since  $w_1$ has non negative coordinates and is normed to $1$,  $\sum_i w_1(i) |w_1|_{\infty} \geq 1=|w_1|_2^2.$ Thus we immediately get the first lower bound $$\langle v_1, w_1\rangle =\frac{1}{\sqrt N} \sum_{i=1}^N w_1(i)\geq \frac{1}{\sqrt N |w|_{\infty}}.$$
Let  $i_o$ be a coordinate such that $w_1(i_0)=|w|_{\infty}.$ Then one has that 
$$\mu_1w_{i_0}=\sum_{j=1}^N P_{i_0j}w_j=\sum_{j=1}^N P_{i_0j}w_{i_0}+\sum_{j=1}^N P_{i_0j}(w_j-w_{i_0}).$$

Fix $\eta>0, \epsilon >0$ that we allow further to depend on $N$ and such that $\eta \gg \epsilon$.
Using that 
$\mu_1\geq d_{\max}(1-\epsilon)$ (see Proposition \ref{Prop1}), 
we thus obtain that
\be \sum_{j=1}^N P_{i_0j}(w_{i_0}-w_j)\leq \epsilon  d_{\max}w_{i_0},\label{boundB}
\ee 
where $d_{\max} = \max_{i} \sum_{j=1}^N P_{i,j} \simeq c_0\frac{N}{\gamma}$.  Define now
  $$B:=\{j, P_{i_0j}>\eta \mathrm{\ and\ }w_j<\frac{w_{i_0}}{2}\}$$
 and 
 $$\overline{B}:=\{j, P_{i_0j}>\eta \mathrm{\ and\ }w_\geq \frac{w_{i_0}}{2}\}$$
 
 Using (\ref{boundB}), one obtains that $\eta \sharp B w_{i_0}/2\leq \epsilon w_{i_0}d_{\max}.$
 This means that \be \label{sharpB}\sharp B \leq \frac{2\epsilon}{\eta} d_{\max}.\ee
 
 We can also deduce from the fact $\mu_1\geq d_{\max}(1-\epsilon)$ that $$\sum_{j=1}^N P_{i_0j}\geq d_{\max}(1-\varepsilon).$$
 Let us assume for the moment  that $$\sum_{j: P_{i_0j}\geq \eta}P_{i_0j}\geq cd_{\max}$$ for some constant $c$. Then by (\ref{sharpB}) this implies 
 $$
 \sharp \overline{B} \geq cd_{\max}- \sharp B  \geq d_{\max} (c - \frac{2\varepsilon}{\eta})\geq Cd_{\max}
 $$
 for some constant $C>0$. Using the fact that $\|w\|=1$, this implies that  $d_{\max}Cw_{i_0}^2/4 \leq 1$ which in turn yields that 
 $$|w|_{\infty}\leq \frac{C'}{\sqrt{d_{\max}} },$$ and then Lemma \ref{lemme:r_1} will be proved.
 
Therefore, it remains to prove that  $$\sum_{j: P_{i_0j}\geq \eta}P_{i_0j}\geq cd_{\max}.$$ This is true if $i_0$ is such that $|X_{i_0}|^2 \leq \frac{\ln \gamma}{\gamma}$, by slightly adapting the proof of Lemma \ref{lem: 1} and choosing $\eta$ of the order of $\min\{ \sqrt{\varepsilon}, 1/\gamma\}$ -- more precisely, the only change in the proof of Lemma \ref{lem: 1},  is the control of $S_1$.

   
   One can easily extend this claim if $|X_{i_0}|^2 \leq \frac{K\ln \gamma}{\gamma}$ for some constant $K$ large enough.
 Now noting $\sum_j P_{ij}^*$ the subsum over those indices $j$ such that $P_{ij}\leq \eta$, one has that
 \begin{align*}
 &\mathbb{P}\left ( \exists i, |X_i|^2\geq \frac{K\ln \gamma}{\gamma}\: \sum_j^*P_{ij}\geq d_{\max}(1-2\epsilon) \right)\cr
 &\leq \mathbb{P}\left ( \exists \:\frac{CN}{\gamma}) \mathrm{\ points\ }X_j \mathrm{\ in\ a\ ball\ } B(x,r), |x|\geq \frac{(K-1)\ln \gamma}{\gamma}, \: r\leq 2\frac{\ln \gamma}{\gamma}\right) . \cr
 &\leq C'' \binom{N}{\frac{N}{\gamma}}e^{-C' N (K-2)\ln \gamma},
 \end{align*}
where $C, C', C''$ are constants and the last follows from Gaussian integration on squares of size $2\frac{\ln \gamma}{\gamma}$ covering $B(0, (K-3)\frac{\ln \gamma}{\gamma})^c$. Choosing $K $ large enough (actually $K=4$ should be enough) yields the result and finishes the proof of the first part Lemma \ref{lemme:r_1}.

\medskip

We now consider the second, more technical point. Let us consider a subset of indices
$I \subset \{1,\ldots,N\}$ to be fixed later and $w_I=\frac{1}{\sqrt{\sharp I}}(w_I(1), \ldots, w_I(N))^t,$ where $w_I(i)=\mathbbm{1}_{i \in I}.$

Then one has 
$$\langle w_I, v_1\rangle = \sqrt{\frac{\sharp I}{N}}\ \mathrm{\ and\ }\ \langle P_1 w_I, w_I\rangle = \frac{1}{I}\sum_{i,j \in I}P_{ij}=:D_I,
$$
where $D_I$ denotes the average inner connectivity (restricted to edges between two vertices from $I$) and it also holds that$\langle P_1 v_1, v_1\rangle=\overline{d}$ where  $\overline{d}$ is the average global connectivity.
We now show that  we can exhibit such a set $I$ such that $\sharp I \geq  \gamma$ and $D_I = \mu_1(1+\smallO(1))$, since we assumed $\frac{N}{\gamma} \sim Np$. Fix $A>0$.
Set  $$I:=\{1\leq i\leq N, \: \|X_i\|^2 \leq A\frac{\gamma }{N}\}.$$ Since $\gamma \ln \gamma/N$ tends to 0, the arguments of the proof of Lemma \ref{lem: 1} can be easily adapted to prove that $\sharp I \geq \gamma A$ with arbitrarily high probability as long as $A\ll \ln \gamma$. Moreover, adapting again the proof of Lemma \ref{lem: 1} (controlling the sum $S_1$ defined there in a similar fashion since we can still approximate  $\mathbf{n_k^{(i)}}$ by $\frac{N\varepsilon}{2\gamma}$ as  $e^{-\frac{\|X_i\|^2}{2}}$ goes to 1), we obtain that $D_I = \mu_1(1+\smallO(1))$. We can do the same to define a vector supported on $\frac{N}{\gamma}$ coordinates instead of $\gamma$.

Consider now the largest entry of $w_1$: let $i$ be such that $w_i=|w_1|_{\infty}.$ Let $\epsilon$ be fixed small so that $\mu_1\geq \frac{N}{2\gamma}(1-\epsilon)$. Let $J$ be the subset 
$$J=\{j, w_1(j)\geq (1-3\epsilon) w_i\}.$$ Then, one has that $\sum_{j\in J}P_{ij}+(1-3\epsilon) (\sum_j P_{ij}-\sum_{j\in J}P_{ij})\geq \frac{N}{2\gamma}(1-\epsilon)$
from which one deduces that $\sum_{j\in J}P_{ij}\geq \frac{2}{3}\frac{N}{2\gamma}.$ In particular this implies that $w_1$ cannot be localized on less than $\frac{N}{\gamma}$ coordinates (and is roughly equally spread on these coordinates). One can also show that the second block of largest entries of $w_1$ has size at least of order 
$\frac{N}{\gamma}$ and entries greater than $|w_1|_{\infty}(1-3\epsilon)^2.$
 Assume $w_1$ is localized on less than $\gamma$ coordinates so that 
$\langle w_1, w_I \rangle \to 0$.\\ 
In the same way we constructed $I$, one can construct at least $\gamma^2/N$  vectors $\hat{v_i}$ whose support are of size $A\frac{N}{\gamma}$ $A>0$  chosen large enough, 2 by 2 disjoint such that 
$$\langle  \hat{v_i}, P  \hat{v_i}\rangle \geq \frac{N}{2\gamma}(1-\epsilon).$$
Let now $\tilde w_1$ be the vector whose coordinates are those of $w_1$ greater than $\eta |w_1|_{\infty}$, with $\eta>0$  chosen small. Because $w_1$ is localized on less than $\gamma$ coordinates, the number of non zero coordinates of $\tilde w_1$ can be written $k \frac{N}{\gamma}$ for some $k\ll \frac{\gamma^2}{N}.$
Let $\epsilon$ be such that $1-3\epsilon=\eta$, so that there must exist  an index $\mathbf{i}\in J$ such that for some $\delta>0$, $$\sum_{j \notin J}P_{ij}\geq \delta \frac{N}{\gamma}.$$
This follows from the fact that $J$ corresponds to a subset of indices of the smallest of the $X_{i}$'s and the nearest neighbors cannot be all in $J$. 
Furthermore, for the same reason there exist at least $\delta' \frac{N}{\gamma}$ such indices $\mathbf{i}$. 
Indeed define for any vertex $j\in J$: $$S_1(j)=\sum_{k \in J}P_{jk}; S_2(j)=\sum_{l \in J^c}P_{jl}.$$ One then has that 
$$\frac{\mu_1}{S_1(j)+S_2(j)}\to 1, \forall j \in J.$$ In all cases one has that $$\frac{\mu_1}{S_1(j)+S_2(j)}\geq 1-\epsilon.$$
Fix $\delta >0$ small. And set $E_{\delta}=\{j \in J, \frac{S_1(j)}{S_1(j)+S_2(j)}\in [\delta, 1-\delta]\}.$ We call $E_{\delta}$ the boundary of $J$. 
For any $i=1, \ldots, k N \gamma^{-1}$ (corresponding to the non zero entries of $\tilde w_1$), consider the ball $B(X_i, \frac{1}{\gamma})$. It is colored green if $\frac{S_2(i)}{S_1(i)+S_2(i)}>1-\delta$. 
It is colored red $\frac{S_1(i)}{S_1(i)+S_2(i)}>1-\delta$. In all other cases, such a ball is colored blue\footnote{Of course, this choice of colours is completely arbitrary and only for illustration purpose} . One can note that the boundary corresponds to blue balls.
We claim that there exists $\delta >0$ small such that the edge $E_{\delta}$ is non empty and furthermore encircles an area in the order of $k\frac{N}{\gamma}.$
 
 To prove this fact, one first remarks that there are green balls. This follows from the fact that we assume the size of the support of $w_1$ is negligible with respect to 
 $\gamma.$ There also exists at least one red ball. Indeed, consider the ball centered at $X_i$ where $w_i=|w_1|_{\infty}.$ One then has that 
 $$\frac{\mu_1}{S_1(i)+S_2(i)}=\frac{S_1(i)}{S_1(i)+S_2(i)} a_1+\frac{S_2(i)}{S_1(i)+S_2(i)}a_2,$$
 where $a_1S_1=\sum_{k \in J}P_{ik}\frac{w_k}{w_i}, \:  a_2S_2=\sum_{l \in J^c}P_{il}\frac{w_l}{w_i}.$
 One deduces that 
 $$\frac{S_1(i)}{S_1(i)+S_2(i)}\geq\frac{ \frac{\mu_1}{S_1(i)+S_2(i)}-\eta}{a_1-a_2},$$
 where $\frac{\mu_1}{S_1(i)+S_2(i)} \leq a_1\leq 1.$ From this one deduces that 
 \[ \frac{S_1(i)}{S_1(i)+S_2(i)}\geq 1-\frac{\epsilon}{1-\eta}.\]
 Choosing $\eta >0$ small enough ($\eta<1/2$) yields that  \[ \frac{S_1(i)}{S_1(i)+S_2(i)}\geq 1-2\epsilon \geq 1-\delta\]
 provided $\delta \geq 2\epsilon.$
Consider two balls intersecting on more than one third of the total area of one ball. This is the case if the center of the second ball is contained in the first one. They cannot be colored green and red provided $2\delta<1/3.$ From this fact we deduce that there necessarily exists an interface of blue balls surrounding the red balls. 
Now $J$ consists of indices corresponding to those in the area encircled by the blue interface (up to an error in the proportion of $\delta$) and some more points
which are necessarily included in red balls centered at some point $X_j, j \in J$. Note that the proportion of those points in $J$ and such red balls cannot exceed $\delta$. The minimal area $A$ to contain $kN\gamma^{-1}$ points is in the order of $A\geq C k\gamma^{-1}$ for some constant $C$.
 Now the total area covered by red balls with some inside points in $J$ defines a domain $D$ whose area is at most in the order of  $ \frac{k}{\gamma}$. Among these a proportion of at most $2\delta$ corresponds to points in $J$. From this we deduce that 
the area encircled by blue balls is at least $c A$ for some constant $c<1.$ 
 Thus one can find at least $K=(k\gamma)^{1/2} $ blue disks whose support are pairwise disjoint and on the frontier of the domain.

As a consequence there exists at least one normalized vector $\hat{v_i}$ such that the supports of $\hat{v_i}$ and $\tilde w_1$ are disjoint. Calling $I_2$ the support of $\hat{v_i}$ one has that there exists a constant $c>0$
\be R_{v_2}:= \sum_{i \in J, \:j \in I_2}  P_{ij}w_1(i)\frac{1}{\sqrt{\sharp I_2}}=\frac{\mu_1}{\sqrt{\sharp I_2}}\sum_{i \in I_2}w_1(i)\geq c \sqrt{\frac{N}{\gamma}}\eta |w_1|_{\infty}\mu_1 .\label{connex}\ee

Now we can construct at least $ K$ such vectors whose support are pairwise disjoint by considering the blue disks. 
 We denote these vectors $\mathbf{v_1}, \ldots, \mathbf{v_{K}}.$
Let then set $$v=\frac{\sum_{i=1}^{K} \mathbf{v_i}}{\sqrt K}.$$
Then because $\langle \mathbf{v_i}, P\mathbf{v_i}\rangle \geq \frac{N}{2\gamma}(1-\epsilon),$ and (\ref{connex}) one can check that 
$$\sup_r \langle rw_1+\sqrt{1-r^2}v, P \left (rw_1+\sqrt{1-r2}v\right )\rangle $$
is achieved for $r_0<1$ such that $$\frac{r_0}{\sqrt{1-r_0^2}}\geq\frac{\mu_1-\frac{N(1-\epsilon)}{2\gamma }}{ \sqrt K c \sqrt{\frac{N}{\gamma}}\eta |w_1|_{\infty}\mu_1}.$$ The denominator is much larger than $\mu_1$ as one can check that $\sqrt K \sqrt{\frac{N}{\gamma}} | w_1|_{\infty}$ does not tend to $0$.
And furthermore this maximum can excede $\mu_1$: this is a contradiction. 

\subsection{Proof of Theorem \ref{Th:mainSBM}} Let us denote by  $\theta_1$ and $\theta_2$ the two eigenvalues that exit the support of the spectral measure of $P_1$.
Now assuming this holds true, an eigenvector associated to such an eigenvalue $\theta$ has necessarily the form:
$$w= R_1 (\theta) (\alpha_1 v_1+\alpha_2 v_2),$$ where $$\alpha_1v_1+\alpha_2v_2 \in \mathrm{Ker}(I +P_0 R_1).$$
Hereabove and in the sequel we denote $R_1$ for $R_1(\theta)$ for the sake of notations.
Using this one deduces that 
\begin{align*} &\alpha_1= -\frac{\lambda_1\langle v_1, R_1 v_2\rangle}{\lambda_1\langle v_1, R_1 v_1 \rangle +1}\alpha_2\cr
\mathrm{\ and\ } & \lambda_1 \lambda_2 \langle v_1, R_1 v_2\rangle^2=(1+\lambda_1\langle v_1, R_1 v_1\rangle)(1+\lambda_2 \langle v_2, R_1 v_2\rangle).
\end{align*}


Then for such an eigenvector setting $a_i= \langle v_i, R_1 v_i\rangle,$ for $i=1, 2$ and $b=\langle v_1, R_1 v_2\rangle$ we obtain that
\be \label{first}\langle w, v_2\rangle^2 = \frac{\alpha_2^2}{\lambda_2^2}; \langle w, v_1\rangle = \frac{b \alpha_2}{1+\lambda_1a_1}.\ee
So far we have not normalized the eigenvector $w$: this has to be considered in order to show that there is indeed some information on $v_2$ using the two normalized eigenvectors. 
Let us  now recall the equation to compute the two eigenvalues $\theta_i$: 
\be \label{impl}f_{\lambda_1, \lambda_2}( \theta)= (1+\lambda_1 a_1(\theta))(1+\lambda_2 a_2(\theta))-\lambda_1 \lambda_2 b^2(\theta)=0,\ee
which we have solved as $\theta$ being a function of $\lambda_1$ and $\lambda_2.$ The very definition of $w$ yields that 
$$||w||^2=\alpha_2^2 \left (\frac{\lambda_1^2 b^2}{(\lambda_1a_1+1)^2}a_1'(\theta) +a_2'(\theta)-2\frac{\lambda_1 b}{\lambda_1 a_1 +1}b'(\theta)\right ) .$$
Using (\ref{impl}) we obtain that 
\be ||w||^2=\alpha_2^2 \frac{\frac{\partial f_{\lambda_1, \lambda_2}}{\partial \theta}}{\lambda_2 (1+\lambda_1a_1)}=\alpha_2^2 \frac{\frac{\partial f_{\lambda_1, \lambda_2}}{\partial \theta}}{\frac{\lambda_1 \lambda_2 b^2}{a_2^2}-(1+\lambda_1 a_1)}, \label{last}\ee
and combining (\ref{first}) and (\ref{last}) gives
\be\label{EQ:correlation} \frac{\langle w, v_2\rangle^2}{||w||^2}= \frac{1}{\frac{\partial f_{\lambda_1, \lambda_2}}{\partial \theta}} \frac{1+\lambda_1a_1}{\lambda_2}.\ee
Notice that  Equation \eqref{EQ:correlation} implies that there are at most two eigenvalues of $P_0+P_1$ that separate from the spectrum of $P_1$; denote them by $\theta_1$ and $\theta_2$. We also recall that we have denoted by $\theta(\lambda_1)$ and $\theta(\lambda_2)$ the respective solutions of $1+\lambda_1a_1=0$ and $1+\lambda_2a_2=0$. We claim that those four specific values satisfy the following relations
$$\begin{array}{l}\theta_2 \leq \min\{ \theta(\lambda_2), \theta(\lambda_1)\}\\ \theta_1 \geq \max\{ \theta(\lambda_2), \theta(\lambda_1)\}\end{array}
, \quad \theta(\lambda_2)\leq \lambda_2+\mu_1 \quad \mathrm{\ and\ } \lambda_1 \leq \theta(\lambda_1)\leq \lambda_1 +\mu_1\ .  
$$The inequalities on the left are a consequence of the fact that $\theta_1$ and $\theta_2$ are solutions of $f_{\lambda_1, \lambda_2}( \theta)=0$ thus $(1+\lambda_1 a_1(\theta_i))$ and $(1+\lambda_2 a_2(\theta_i))$ must have the same sign, the one of $\frac{\partial f_{\lambda_1, \lambda_2}}{\partial \theta}(\theta_i)$. The second inequality is a consequence of the fact that $|\mu_j|\leq \mu_1$ and then plugging this value in $a_2$. The inequalities on the right are a consequence of the very last argument and of the  fact that $\theta(\lambda_1) \geq \lambda_1$ since $\theta(\lambda_1)$ is an eigenvalue of $P_0+\lambda_1v_1v_1^\top$.

This immediately gives the first bound
\begin{equation}\label{EQ:numerator}
-(1+\lambda_1a_1(\theta_2))  = \lambda_1 \sum_j \frac{r_j^2}{\theta_2-\mu_j}-1 \geq \frac{\lambda_1}{\lambda_2+2\mu_1}-1
\end{equation}

As a consequence, it remains to control $\frac{\partial f_{\lambda_1, \lambda_2}}{\partial \theta}(\theta_2)$. Notice that, by definition of $ f_{\lambda_1, \lambda_2}$ and the fact that $ f_{\lambda_1, \lambda_2}(\theta_2)=0$, we get
\begin{align*}
\left|\frac{\partial f_{\lambda_1, \lambda_2}}{\partial \theta}(\theta_2)\right| \leq \lambda_1\frac{\partial a_1}{\partial \theta}(\theta_2)\big(\lambda_2|a_2|-1\big)+\lambda_2\frac{\partial a_2}{\partial \theta}(\theta_2)\big(\lambda_1|a_1|-1\big)\\\hspace{2cm}+2\frac{\partial b}{\partial \theta}(\theta_2)
\sqrt{\lambda_1\lambda_2}\sqrt{(1+\lambda_1a_1)(1+\lambda_2a_2)}
\end{align*}
Moreover, we immediately get the following upper-bounds
$$
|a_i(\theta)|=\sum_j \frac{r_j^2}{\theta-\mu_j} \leq \frac{1}{\theta-\mu_1}, \quad |a_2(\theta)| \leq \frac{1}{\theta-\mu_1}, \quad a'_1,a'_2,b'\leq \frac{1}{(\theta-\mu_1)^2}\ .
$$
Plugging those estimates in $\frac{\partial f_{\lambda_1, \lambda_2}}{\partial \theta}(\theta_2)$ gives that
\begin{align}\label{EQ:denominator}
\nonumber \lambda_2\left|\frac{\partial f_{\lambda_1, \lambda_2}}{\partial \theta}\right| \leq \frac{\lambda_1\lambda_2}{(\theta_2-\mu_1)^2}\big(\frac{\lambda_2}{\theta_2-\mu_1}-1\big)+\frac{\lambda_2^2}{(\theta_2-\mu_1)^2}\big(\frac{\lambda_1}{\theta_2-\mu_1}-1\big)\\\hspace{2cm}+2\frac{\sqrt{\lambda_1\lambda_2}\lambda_2}{(\theta_2-\mu_1)^2}\sqrt{\big (\frac{\lambda_2}{\theta_2-\mu_1}-1\big  )\big  (\frac{\lambda_1}{\theta_2-\mu_1}-1\big )}
\end{align}
From Equation \eqref{EQ:separation}, we get that $\theta_2 \geq \frac{\lambda_2}{4}\geq \mu_1(1+\varepsilon)$ so that we get  non-zero correlation between $w_2$ and $v_2$ from Equations \eqref{EQ:numerator} and \eqref{EQ:denominator}.

\medskip

We can actually be more precise. It is indeed quite easy to prove using \eqref{functg} that 
$$
f_{\lambda_1,\lambda_2}(\theta) \geq 1+\frac{\lambda_1}{\mu_1 -\theta}+\frac{\lambda_2}{\mu_1 -\theta}+\frac{\lambda_1\lambda_2}{(\mu_1 + \theta)^2}.
$$

Let us assume that the ratios $\frac{\lambda_1}{\lambda_2}=q> 1$ and $0\leq \frac{\mu_1}{\lambda_2}=x \leq 1$ are fixed, and make the change of variables
 $\theta = \lambda_2 - \gamma \mu_1=(1-\gamma x)\lambda_2$, so that $$
f_{\lambda_1,\lambda_2}(\theta) \geq 1-\frac{1+q}{1-(\gamma+1)x}+\frac{q}{(1-(\gamma-1)x)^2}.
$$
In order to control the solution of $f_{\lambda_1,\lambda_2}=0$ w.r.t.\ $\gamma$, we are  going to assume for the moment that $(\gamma+1)x \leq \frac{1}{2}$ so that the r.h.s.\ can be easily lower-bounded into 
\begin{align*}
f_{\lambda_1,\lambda_2}(\theta) & \geq 1-(1+q)\big(1+(\gamma+1)x+2((\gamma+1)^2x^2)\big)\\ & \hspace{1cm} +q\big(1+2(\gamma-1)x-(\gamma-1)^2x^2\big)\\ 
&= x\Big( \big[\gamma(q-1)-(3q+1)\big]-2x\big[(3q+1)\gamma^2-2(q-1)\gamma +(3q+1)\big]\Big),
\end{align*}
which gives an explicit (and uniformly bounded) upper-bound  $\overline{\gamma}$ for $\gamma$, i.e., the solution of the above degree 2 polynomial. Notice that when $x$ goes to zero, the expression boils down to
$$\overline{\gamma} =  3 +\frac{4}{q-1}+\mathcal{O}(x).$$
Plugging $\overline{\gamma}$   into Equations \eqref{EQ:numerator} and \eqref{EQ:denominator} gives that 
$$\frac{|\langle w, v_2\rangle|^2}{\|w\|^2} \geq \Big(1-\frac{2x}{q-1}\Big)\frac{(1-(\overline{\gamma}+1)x)^3}{\Big(1+ \frac{\overline{\gamma}+1}{2(q-1)}x +\sqrt{q(\overline{\gamma}+1)x}\Big)^2}\ 
$$which is uniformly bounded away from 0. 

Moreover, when $x$ goes to 0, it holds that 
\begin{align*}
\frac{|\langle w, v_2\rangle|}{\|w\|} &\geq 1 - 2\frac{q}{\sqrt{q-1}}\sqrt{x} - \mathcal{O}(x) 
\\ &= 1-2\frac{\frac{\lambda_1}{\lambda_2}}{\sqrt{\frac{\lambda_1}{\lambda_2}}-1}\sqrt{\frac{\mu_1}{\lambda_2}}-\mathcal{O}\big (\frac{\mu_1}{\lambda_2}\big )
\end{align*}
and when $x$ is small enough\footnote{Numerical implementation suggests that those computations  hold for $x \leq \frac{q-1}{8q}$, i.e., when the value on $\overline{\gamma}$ is set to $3+\frac{4}{q-1}$ without the $\mathcal{O}(x)$ term.}, then we also have that $(\gamma+1)x \leq \frac{1}{2}$ as required.
This proves the theorem (since  ratios are assumed to be uniformly lower and upper-bounded).

\end{appendix}

\end{document}